\documentclass{article}

\usepackage[nonatbib,final]{neurips_2020}

\usepackage[square,numbers]{natbib}
\usepackage[utf8]{inputenc} 
\usepackage[T1]{fontenc}    
\usepackage{hyperref}       
\usepackage{url}            
\usepackage{booktabs}       
\usepackage{amsfonts}       
\usepackage{nicefrac}       
\usepackage{microtype}      

\usepackage{amssymb}
\usepackage{amsthm}
\usepackage{amsfonts}
\usepackage{dsfont}
\usepackage{mathtools}

\usepackage{diagbox}
\usepackage{multirow}
\usepackage{graphicx}
\usepackage{epsfig}
\usepackage{subfig}

\usepackage[linesnumbered, vlined, ruled]{algorithm2e}
\usepackage{mdframed}
\usepackage{xcolor}
\newcommand{\highlight}[1]{ \colorbox{black!20}{$#1$} }

\usepackage{appendix}
\usepackage{comment}
\usepackage{paralist}
\usepackage{enumitem}

\newtheorem{definition}{Definition}{\itshape}{\rmfamily}
\newtheorem{theorem}{Theorem}{\itshape}{\rmfamily}
\newtheorem{lemma}{Lemma}[theorem]{\itshape}{\rmfamily}
{\itshape}{\rmfamily}
\newtheorem{proposition}{Proposition}
{\itshape}{\rmfamily}

{\itshape}{\rmfamily}

\newenvironment{proofidea}{\bigskip \par{\noindent \emph{Proof idea:}}}{\qed\par}

\newcommand{\squishlist}{
	\begin{list}{$\bullet$}
	{ 	\setlength{\itemsep}{0pt}      \setlength{\parsep}{1pt}
		\setlength{\topsep}{3pt}       \setlength{\partopsep}{0pt}
		\setlength{\leftmargin}{1.5em} \setlength{\labelwidth}{1em}
		\setlength{\labelsep}{0.5em} } }

\newcommand{\squishlisttwo}{
	\begin{list}{$\bullet$}
		{ \setlength{\itemsep}{0pt}    \setlength{\parsep}{0pt}
			\setlength{\topsep}{0pt}     \setlength{\partopsep}{0pt}
			\setlength{\leftmargin}{2em} \setlength{\labelwidth}{1.5em}
			\setlength{\labelsep}{0.5em} } }
	
\newcommand{\squishend}{\end{list}}

\newcommand{\eq}[2]{\begin{equation}\label{#1} #2 \end{equation}}
\newcommand{\eqm}[2]{\begin{equation}\label{#1}\begin{split} #2 \end{split}\end{equation}}
\newcommand{\eqa}[2]{\begin{align}\label{#1} #2 \end{align}}
\newcommand{\eqmx}[1]{\begin{equation*}\begin{split} #1 \end{split}\end{equation*}}

\newcommand{\define}[0]{\doteq}

\newcommand{\ra}[0]{\rightarrow}

\newcommand{\limit}[1]{\lim\limits_{#1}}

\newcommand{\Op}[2]{#1 \Big[ #2 \Big]}
\renewcommand{\P}{\operatorname*{\mathbb{P}}}
\renewcommand{\Pr}[1]{\P[#1]}
\newcommand{\E}[1]{\operatorname*{\mathbb{E}}_{\substack{#1}}}
\newcommand{\Exp}[2]{\Op{\E{#1}~}{#2}}

\newcommand{\indicator}[1]{\mathds{1}(#1)}
\newcommand{\Real}[0]{\mathbb{R}}

\renewcommand{\*}{\cdot}

\newcommand{\AS}{\mathcal{A}}
\renewcommand{\SS}{\mathcal{S}}

\newcommand{\T}{P}
\newcommand{\R}{R}
\newcommand{\ga}{\gamma}

\newcommand{\absorb}{s_{\scriptscriptstyle \Box}}
\newcommand{\terminals}{\SS_\perp}
\newcommand{\step}[1]{^{\scriptscriptstyle (#1)}}

\newcommand{\U}{\mathcal{U}}
\newcommand{\M}{\mathcal{M}}
\newcommand{\gammac}{{\scriptscriptstyle \ga_c}}
\newcommand{\D}{\mathcal{D}}
\newcommand{\J}{\tilde{J}}
\newcommand{\proofvspace}{\vspace{-0.15in}}

\title{Steady State Analysis of Episodic Reinforcement Learning}

\author{%
	Huang Bojun \\
	Rakuten Institute of Technology, Tokyo, Japan \\
	\texttt{bojhuang@gmail.com} \\
}

\begin{document}
\maketitle

\begin{abstract}
	This paper proves that the episodic learning environment of every finite-horizon decision task has a unique \emph{steady state} under any behavior policy, and that the marginal distribution of the agent's input indeed converges to the steady-state distribution in essentially all episodic learning processes. This observation supports an interestingly reversed mindset against conventional wisdom: While the existence of unique steady states was often presumed in continual learning but considered less relevant in episodic learning, it turns out their existence is \emph{guaranteed} for the latter. Based on this insight, the paper unifies episodic and continual RL around several important concepts that have been separately treated in these two RL formalisms. Practically, the existence of unique and approachable steady state enables a general way to collect data in episodic RL tasks, which the paper applies to policy gradient algorithms as a demonstration, based on a new \emph{steady-state policy gradient theorem}. Finally, the paper also proposes and experimentally validates a perturbation method that facilitates rapid steady-state convergence in real-world RL tasks.
\end{abstract}


\section{Introduction}
\label{sec_introduction}

Given a decision task, reinforcement learning (RL) agents iteratively optimize decision policy based on empirical data collected from continuously interacting with a learning environment of the task. To have successful learning, the empirical data used for policy update should represent some desired distribution. Every RL algorithm then faces two basic questions: \emph{what} data distribution shall be considered desired for the sake of learning, and \emph{how} should the agent interact with the environment to actually obtain the data as desired? Depending on the type of the given task, though, existing RL literature have treated this data collection problem in two different ways.

In \emph{continual reinforcement learning}, the agent immerses itself in a single everlasting sequential-decision episode that is conceptually of infinite length (or of life-long length). In this case, a \emph{stationary (or steady-state) distribution} is a fixed-point distribution over the agent's input space under the transition dynamic induced by a decision policy. The concept of steady state has been pivotal in continual RL literature. It is typically \emph{presumed} that a unique stationary distribution exists when rolling out any policy in the continual RL model being studied, and that the empirical data distribution of the rollout indeed converges over time to this stationary distribution due to (again assumed) ergodicity of the system dynamic~\cite{2002:tsitsiklis, 2012:opac, 2018:RL, 2019:generalized}. Continual RL algorithms can be derived and analyzed by examining system properties under the steady state~\cite{2001:continual_pg, 2002:tsitsiklis, 2014:thomas, 2015:YU}, and many of the resulted algorithms require the training data to follow the stationary distribution of some behavior policy~\cite{2012:opac, 2018:RL} (or a mixture of such distributions~\cite{2015:dqn}), which can then be efficiently collected from a few (or even a single) time steps once the rollout of the behavior policy converges to its steady state.

The situation has been quite different for \emph{episodic reinforcement learning}, in which the agent makes a finite number of decisions before an episode of the task terminates. Episodic RL tasks account for the vast majority of experimental RL benchmarks and of empirical RL applications at the moment~\cite{2016:gym, 2017:time_limit}. Due to the finiteness of decision horizon, for episodic tasks steady state was widely considered non-existent~\cite{2014:thomas} or having a degenerate form~\cite{2018:RL}, in either case irrelevant. The lack of meaningful steady state in such tasks has in turn led to more modeling disparities in the episodic RL formalism.  
Indeed, as \citet{2018:RL} (chap. 9, p. 200) wrote: "\emph{[t]he two cases, continuing and episodic, behave similarly, but with [function] approximation they must be treated separately in formal analyses, as we will see repeatedly in this part of the book}". In particular, the desired data distributions for episodic RL algorithms are usually characterized by alternative concepts, such as the expected \emph{episode-wise visiting frequencies}~\cite{2018:RL}, or a $\gamma$-discounted variant of it~\cite{2014:dpg, 2015:trpo}.

Intriguingly, despite the rather different theoretical framework behind, many episodic RL algorithms~\cite{2016:a3c,2018:RL} update policies using data from a small time slice (such as from one time step~\cite{2018:RL} or from several consecutive steps~\cite{2016:a3c, 2017:ppo}) of the episodic learning process, in almost the same way as how continual RL algorithms~\cite{1989:qlearning, 2016:a3c} collect data from ergodic models. This ``online-style'' data collection~\cite{2016:benchmarking}, while being popular in modern RL algorithms\cite{2016:a3c, 2017:ppo, 2015:dqn, 2015:ddpg,2018:sac} thanks to its simplicity and sample efficiency, is often not well justified in episodic tasks by the prescribed data distribution which requires collecting a whole episode trajectory as \emph{one} sample point of the distribution~\cite{2016:benchmarking}.
 
In this paper, we propose a new perspective to treat the data collection problem for episodic RL tasks. We argue that in spite of the finite decision horizon, the RL agent has the chance to \emph{repeat} the same decision process in the \emph{learning environment} of an episodic RL task. The resulted learning process is an infinite loop of homogeneous and finite episodes. By formalizing such an episodic learning environment into a family of infinite-horizon MDPs with special structures (Section \ref{sec_elp}), we mathematically proved that a unique stationary distribution exists in every such episodic learning process, regardless of the agent policy, and that the marginal distribution of the rollout data indeed converges to this stationary distribution in ``proper models'' of the episodic RL task (Section \ref{sec_ergodic}). 

Conceptually, our theoretical observation supports a reversed mindset against conventional wisdom: while unique steady state is (only) presumed to exist in continual tasks, it is now \emph{guaranteed} to exist in all episodic tasks. Moreover, through analyzing the Bellman equation under the steady state, we obtained interesting and rigorous connections between the separately defined semantics of important concepts (such as \emph{performance measure} and \emph{on-policy distribution}) across episodic and continual RL. These results help with unifying the two currently disparate RL formalisms (Section \ref{sec_conceptual}).

Algorithmically, the unified framework enables us to collect data efficiently in episodic tasks in a similar way as in the continual setting: Write the quantity to be computed as an expectation over the stationary distribution of some policy $\beta$, rollout $\beta$ and wait for it to converge, after which we obtain an unbiased sample point in every single time step. As a demonstration of this general approach, we derived a new \emph{steady-state policy gradient theorem}, which writes policy gradient into such an expected value. The new policy gradient theorem not only better justifies the popular few-step-sampling practice used in modern policy gradient algorithms~\cite{2016:a3c, 2018:RL}, but also makes explicit some less explored bias in those algorithms (Section \ref{sec_pg}). Finally, to facilitate \emph{rapid} steady-state convergence, we proposed a \emph{perturbation trick}, and experimentally showed that it is both necessary and effective for data collection in episodic tasks of standard RL benchmark (Section \ref{sec_perturb}).

\vspace{-0.04in}
\section{Preliminaries}
\vspace{-0.05in}
\label{sec_preliminaries}

\textbf{Markov chain}. 
In this paper, a \emph{markov chain} is a homogeneous discrete-time stochastic process with countable (finite or infinite) state space. The state-transition probabilities are written into a \emph{transition matrix} $M$, where $M(s,s')$ is the entry of row $s$ and column $s'$ which specifies $\Pr{s_{t+1}=s'|s_{t}=s}$. A \emph{rollout} of $M$ generates a trajectory $\zeta=(s_0,s_1,s_2,\dots)$ of infinite length. Let row vector $\rho\step{t}$ denote the marginal distribution of the state at time $t$, so $s_t \sim \rho\step{t} = \rho\step{t-1} M$. The \emph{limiting distribution} of $s_t$, if exists, is the marginal distribution at infinite time, i.e. $\lim_{t \ra \infty} \rho\step{t}$, and the \emph{stationary (or steady-state) distribution} is defined as the fixed-point distribution with $\rho= \rho M$. 

The existence and uniqueness of stationary and limiting distributions are characterized by the following well-known concepts and conditions: Given a markov chain with transition matrix $M$, a state $s$ is \emph{reachable} from a state $\bar{s}$ if $\Pr{s_t=s|s_0=\bar{s}}>0$ for some $t\geq 1$, and the markov chain $M$ is said \emph{irreducible} if every state is reachable from any state in $M$. The \emph{mean recurrence time} of a state $s$ is the expected number of steps for $s$ to reach itself, denoted as $\E{}[T_s] \define \sum_{t=1}^{\infty} t \* \Pr{s_t=s \texttt{~and~} s \not\in \{s_{1:t-1}\}|s_0=s}$, and state $s$ is said \emph{positive recurrent} if $\E{}[T_s]<\infty$. The markov chain $M$ is positive recurrent if every state in $M$ is positive recurrent. Finally, a state $s$ is \emph{aperiodic} if $s$ can reach itself in two trajectories with co-prime lengths, i.e. if $\gcd\{t>0:\Pr{s_t=s|s_0=s}>0\}=1$. 
\begin{proposition}
	\label{thm_mc_stationary}
	An irreducible markov chain $M$ has a unique stationary distribution $\rho_{\scriptscriptstyle M}=1/\E{}[T_s]$ if and only if $M$ is positive recurrent. 
	(\cite{2009:markov_chain}, Theorem 54)
\end{proposition}

\begin{proposition}
	\label{def_ergodic}
	An irreducible and positive-recurrent markov chain $M$ has a limiting distribution $\lim\limits_{t \ra \infty}\rho\step{t}=\rho_{\scriptscriptstyle M}$ if and only if there exists one aperiodic state in $M$.
	(\cite{2009:markov_chain}, Theorem 59)
\end{proposition}
\vspace{-0.05in}
A markov chain satisfying the condition in Proposition \ref{def_ergodic} is called an \emph{ergodic} markov chain.

\textbf{Markov Decision Process (MDP)}.
An MDP $\M=(\SS, \AS, \R, \T, \rho_0)$ is a sequential decision making model where $\SS$ is a countable state space, $\AS$ a countable action space, $\R(s)\in [r_{min},r_{max}]$ is a reward assigned to each state $s\in \SS$, $\T(s'|s,a) \define \Pr{s_{t+1}=s' | s_t=s, a_t=a}$ specifies action-conditioned transition probabilities between states, and $\rho_0$ is the initial distribution with $s_0\sim \rho_0$.  

A policy function $\pi: \SS\times \AS \to [0,1]$ prescribes the probability to take action under each state. With $\Pr{a_t=a|s_t=s} = \pi(s,a)$, every policy $\pi$ induces a markov chain with transition matrix $M_{\pi} (s,s') = \sum_{a} \pi(s,a) \T(s'|s,a)$. Accordingly, row vectors $\rho\step{t}_{\pi}$ and $\rho_{\pi}$ denote the marginal and stationary distributions of $s_t$ under policy $\pi$, respectively, so $\rho\step{t+1}_{\pi} = \rho\step{t}_{\pi} M_{\pi}$, and $\rho_{\pi} = \rho_{\pi} M_{\pi}$. 

An MDP is ergodic if for every policy $\pi$ the induced markov chain $M_\pi$ is ergodic, in which case the \emph{steady-state performance}~\cite{2002:tsitsiklis} of a policy $\pi$ is $J_{avg}(\pi) \define \lim\limits_{ T \ra \infty} \frac{1}{T} \sum_{t=1}^{T} \R(s_t) = \E{s\sim \rho_\pi}[\R(s)]$. 
When a set of \emph{terminal states} $\terminals \subset \SS$ is identified, a rollout trajectory $\zeta$ is said \emph{terminated} when it reaches a terminal state. We use $T(\zeta)\define \inf\{t\geq 1:s_t\in \terminals\}$ to denote the \emph{termination time} of $\zeta$,
and the \emph{episode-wise performance} of a policy $\pi$ is defined as $J_{epi}(\pi) \define \E{\zeta\sim M_\pi}[\sum_{t=1}^{T(\zeta)} \R(s_t)]$. 

A value function $Q: \SS \times \AS \to \Real$ prescribes the ``value'' of taking an action under a state. 
In this paper, we consider the following \emph{product-form family} of value functions~\cite{2011:horde,2011:insights,2012:opac,2014:qlambda,2016:emphatic,2017:predictron,2017:unifying} :
\vspace{-0.05in}
\eqa{}{
	Q_\pi(s,a) 
	&\define \E{s_1\sim \T(s,a), \{s_{\geq 2}\}\sim M_\pi} [\sum_{t = 1}^\infty \R(s_{t}) \* \prod_{\tau=1}^{t-1} \gamma(s_\tau)] \label{q_pi} \\
	&= \E{s'\sim \T(s,a), a'\sim \pi(s')} [\R(s')+\gamma(s') \* Q_\pi(s',a')] \label{bellman}
}
in which $\T(s,a)$ and $\pi(s')$ are short-hands for the conditional distributions $\T(\cdot|s,a)$ and $\pi(\cdot|s')$, and $\gamma: \SS \to [0,1]$ is called the \emph{discounting function}. Besides entailing the \emph{Bellman equation} \eqref{bellman}, the definition of $Q_{\pi}$, i.e. \eqref{q_pi}, also induces the \emph{state-value function} $V_\pi(s) \define \E{a\sim \pi(s)}[Q_\pi(s,a)]$.

With $\ga(s) = \ga_c$ for constant $\ga_c \in [0,1)$, the product-form value functions $Q_\pi$ and $V_\pi$ subsume the classic value functions that have been underlying much of the existing RL literature, and we will use $Q^\gammac_\pi$ and $V^\gammac_\pi$ to refer to this particular form of value functions with constant discounting function. 
On the other hand, with $\ga(s)=\indicator{s\not\in \terminals}$, the episode-wise performance $J_{epi}$ can be recast as the product-form value of the initial state~\cite{2011:horde}: $J_{epi}(\pi)=V_\pi(s_0)$ if $\ga(s)=1$ for $s\not\in \terminals$ and $\ga(s)=0$ for $s\in \terminals$. We will call the function $\ga(s)=\indicator{s\not\in \terminals}$, the \emph{episodic discounting function}. See Appendix \ref{sec_mdp} for more discussions on the formulation choices behind the MDP model.

\textbf{Reinforcement Learning (RL)}. 
In continual RL, we are given a decision task modeled by an ergodic MDP $\M_D$, and the goal is to find good policy with respect to the steady-state performance $J_{avg}$ through a single continual rollout of $\M_D$. In episodic RL, the decision task $\M_D$ is a finite-horizon MDP with terminal states, and the goal is to optimize the episode-wise performance $J_{epi}$ through repeatedly rolling out $\M_D$ from the beginning. In this case, a special resetting procedure will intervene the learning process to start a new decision episode upon the end of the last~\cite{2016:gym}.

A common and basic idea, for both continual and episodic RL algorithms, is to use the rollout data to optimize a parameterized policy function $\pi(s,a;\theta)$ with respect to a surrogate objective $\J(\theta)$ via the stochastic gradient method. At a update time $t$, the gradient $\nabla_\theta \J$ is approximately computed as a sample mean of some computable function $F$ over a ``mini-batch'' $\D_t$, with
$
\nabla_\theta \tilde{J}(\theta_t) \approx \frac{1}{|\mathcal{D}_t|}\sum_{(s,a)\in \mathcal{D}_t} F(s,a,\theta_t)
$
, where $\D_t$ is a selected subset of all the rollout data up to time $t$. 

For example, \emph{policy gradient algorithms}, as a family of widely used RL algorithms, work on policy functions $\pi(s,a;\theta)$ that are directly differentiable, and typically choose $\J_{pg}(\theta) \define \E{s_0\sim \rho_0}[V^\gammac_{\scriptscriptstyle \pi(\theta)}(s_0)]$
~\footnote{
	To simplify notations we will write $\rho_\theta$, $Q_\theta$, $V_\theta$, $J(\theta)$ for $\rho_{\pi(\theta)}$, $Q_{\pi(\theta)}$, $V_{\pi(\theta)}$, $J(\pi(\theta))$ from now on.
} 
in episodic tasks as the surrogate objective, whose gradient can be estimated by $F_{pg}(s,a,\theta) \define \widehat{Q^\gammac_\theta}(s,a) \nabla \log\pi(s,a;\theta)$, where $\widehat{Q^\gammac_\theta}(s,a)$ is some practical estimation of $Q^\gammac_\theta(s,a)$ for the specific $(s,a)\in \D_t$. 
Substituting $\tilde{J}_{pg}$ and $F_{pg}$ to the general RL algorithmic framework above, yields 
\eqa{}{
	\nabla J_{epi}(\theta) 
	\approx \nabla \J_{pg}(\theta) 
	&= \sum_{\tau=0}^\infty ~(\ga_c)^\tau  
	\Exp{s_\tau\sim \rho\step{\tau}_{\theta},~ a_\tau\sim\pi(s_\tau;\theta)}{
		Q^\gammac_\theta(s_\tau,a_\tau) \* \nabla \log\pi(s_\tau,a_\tau;\theta) 
	}  \label{discounted_pg1}\\
	&\approx \frac{1}{|\mathcal{D}_t|}\sum_{(s,a)\in \mathcal{D}_t} 
	\widehat{Q^\gammac_\theta}(s,a) \* \nabla \log\pi(s,a;\theta) \label{discounted_pg2}
}
where the equality part in \eqref{discounted_pg1} is known as the classic \emph{policy gradient theorem}~\cite{2000:pg}.

Policy gradient algorithms illustrate a general disparity observed in episodic RL between the ``desired'' data distribution and the data actually collected. Specifically, according to \eqref{discounted_pg1}, the policy gradient should be estimated by rolling out an episode using $\pi(\theta)$ then summing over $(\ga_c)^\tau \cdot F_{pg}$ across all steps $\tau$ in the episode, which collectively provide one sample point to the right-hand side of \eqref{discounted_pg1}. Many policy gradient algorithms indeed work this way~\cite{2016:benchmarking, 2015:trpo}. However, some modern policy gradient algorithms, such as A3C~\cite{2016:a3c} or the variant of REINFORCE as described in \cite{2018:RL}, compute the policy gradient based on data only from a small time window, much smaller than the average episode length. Such an ``online'' approach~\cite{2016:benchmarking} has witnessed remarkable empirical success in practice~\cite{2016:a3c}.

A popular explanation for the disparity above is to re-arrange the sum $\sum_{\tau=0}^\infty ~(\ga_c)^\tau  \sum_{s_\tau} \rho\step{\tau}_\theta(s_\tau)$ in \eqref{discounted_pg1} into $\sum_s \big( \sum_{\tau=0}^\infty (\ga_c)^\tau \rho\step{\tau}_\theta(s) \big)$, which is a weighted sum over states, thus can be interpreted as an expectation over a state distribution that we will call it the \emph{episode-wise visitation distribution} $\mu_\pi^\gammac(s) \define \sum_{\tau=0}^\infty (\ga_c)^\tau \rho\step{\tau}_\pi(s)/Z$~\cite{2014:dpg,2015:trpo,2018:RL}, where $Z$ is a normalizing term. However, the visitation-distribution interpretation cannot explain why a \emph{single} rollout data point $s_t$ can serve as a faithful sample of $\mu^\gammac_\pi$. For $\ga_c<1$, we know that $s_t$ is just biased to $\mu^\gammac_\pi$~\cite{2014:thomas}; for $\ga_c=1$, we have $\mu^1_\pi \propto \E{\pi}[\sum_t \indicator{s_t=s}]$, which justifies the practice of taking long-run average across multiple episodes~\cite{2016:benchmarking, 2015:trpo}, but still cannot explain why a single step can represent this long-run average~\cite{2016:a3c, 2018:RL}.  

Importantly, we remark that the aforementioned disparity between desired data and actual data is observed not only in policy gradient algorithms, but also in other \emph{policy-based} algorithms like PPO\cite{2017:ppo}, and in \emph{value-based} and \emph{off-policy} algorithms like Q-Learning\cite{1989:qlearning, 2018:RL} too. The disparity between theory and practice in data collection is currently a \emph{general} issue in the RL framework thus discussed (which subsumes all these algorithms, see Appendix \ref{sec_rl} for more details).


\section{Modeling reinforcement learning process with MDP}
\label{sec_elp}

During the rollout of a policy $\pi$, the state $s_t$ follows the marginal distribution $\rho\step{t}_\pi$ which is dynamically evolving as the time $t$ changes. To reliably use such dynamic data to represent whatever \emph{fixed} distribution as desired, we must understand \emph{how} the marginal state distribution $\rho_\pi\step{t}$ will evolve over time in reinforcement learning tasks. The first step, however, is to have a formalism to the actual environment that the rollout data $s_t$ come from and the marginal distributions $\rho_\pi\step{t}$ reside in. 

For episodic tasks, it is important to distinguish the \emph{learning environment} of a task, denoted by $\M_L$, from the decision environment of the task, denoted by $\M_D$ -- the latter terminates in finite steps, while the former is an infinite loop of such episodes. This difference seems to be recognized by the community~\cite{2018:RL}, yet the tradition mostly focused on formulating the decision environment $\M_D$ as MDP. The multi-episode learning environment $\M_L$, as \emph{the} physical source where learning data really come from for all episodic RL tasks, has received relatively little treatment in the literature. 

We adopt an approach similar (in spirit
~\footnote{
	Technically speaking, our approach is actually closer to an idea that \cite{2017:unifying} argues \emph{against}; see \ref{sec_mdp} for details.}
) with \cite{2017:unifying} to formulate reinforcement learning environments. In general, the \emph{learning process} of a finite-horizon MDP $\M_D$ is an infinite-horizon MDP $\M_L$. We require that a learning process always starts from a terminal state $s_0\in \terminals$, followed by a finite steps of rollout until reaching another terminal state $s_{T_0}$. At $T_0$, the learning process gets implicitly reset into the second episode no matter what action $a_{T_0}$ is taken under state $s_{T_0}$, from there the rollout will reach yet another terminal state $s_{T_0+T_1}$ after another $T_1$ steps. 
In contrast to the case in $\M_D$, a terminal state of the learning MDP $\M_L$ does not terminate the rollout but serves as the initial state of the next episode, thus has identical transition probabilities with the initial state. Also note that in $\M_L$ a terminal state $s\in \terminals$ is just a normal state that the agent will spend one real step on it in physical time, from which the agent receives a real reward $R(s)$~\cite{2016:gym}; different terminal states may yield different rewards (e.g. win/loss). There is no state outside $\SS$, and in particular $\terminals\subseteq \SS$. 

As the learning process of episodic task is just a standard MDP, all formal concepts about MDP introduced in Section \ref{sec_preliminaries}
also apply to episodic learning process. We will use the product-form value function \eqref{q_pi} to analyze the learning process, which with $\ga(s)=\indicator{s\not\in \terminals}$ still recovers the episode-wise performance; specifically, $J_{epi}(\pi)=V_\pi(s_\perp)$, where $s_\perp$ can be any terminal state. 

However, the \emph{interpretations} and roles of  $\rho\step{t}_\pi$, $\rho_\pi$ and $J_{avg}$ have completely changed when we shifted from $\M_D$ to $\M_L$ (which is our main purpose to shift to $\M_L$ in the first place). To prepare theoretical analysis, we need to formally define \emph{episodic learning process} using two natural conditions.

\vspace{0.06in}
\begin{definition}
\label{def_episodic}
	A MDP $\M_L=(\SS,\AS,\R,\T,\rho_0)$ is an \textbf{episodic learning process} if 
	\begin{enumerate}[leftmargin=0.15in] 
		\item (Homogeneity condition): All initial states have identical and action-agnostic transition probabilities; formally, $\forall s, s' \in \SS, a,a'\in \AS$, $\T(s, a) = \T(s', a')$ if $\rho_0(s) \cdot \rho_0(s')>0$. 
		
		Let $\SS_\perp \subseteq \SS$ be the set of all states sharing the same transition probabilities with the initial states; that is, $\terminals$ is the maximal subset of $\SS$ such that $\{s:\rho_0(s)>0\} \subseteq \terminals$ and that $\forall s,s'\in \terminals, a,a'\in \AS, \T(s,a)=\T(s',a')$. A state $s\in\terminals$ is called a \textbf{terminal state}. A finite trajectory segment $\xi=(s_t,s_{t+1},\dots,s_{t+T})$ is called an \textbf{episode} iff $s_t, s_{t+T} \in \terminals$ and $s_\tau \not\in \SS_\perp$ for $t<\tau< t+T$, in which case the \textbf{episode length} is the number of transitions, denoted $|\xi|\define T$.  
		
		\item (Finiteness condition): All policies have finite average episode length; formally, for any policy $\pi$ let $\Xi_\pi=\{\xi:\xi \text{~is episode,~} \P_{\pi}(\xi)>0\}$ be the set of admissible episodes under $\pi$, then $\sum_{\xi\in\Xi_\pi} \P_{\pi}(\xi)=1$ (so that $\P_{\pi}$ is a probability measure over $\Xi_\pi$
		~\footnote{
			Note that the default sample space, i.e. the set of all infinite rollout trajectories $\{\zeta\}$, is a uncountable set (even for \emph{finite} state space $\SS$), while the new sample space $\Xi_\pi$ is a countable set.}
		) and 
		$\sum_{\xi\in\Xi_\pi} \P_{\pi}(\xi) \* |\xi| < +\infty$.
	\end{enumerate} 
\end{definition}

As a familiar special case, an RL environment with time-out bound has a maximum episode length, thus always satisfies the finiteness condition (and having proper time-out bound \emph{is} important and standard practice~\cite{2017:time_limit}). Another example of episodic learning process is the MDPs that contain self-looping transitions but the self-looping probability is less than $1$; in that case there is no maximum episode length, yet the finiteness condition still holds. On the other hand, the learning process of a pole-balancing task without time-out bound is \emph{not} an episodic learning process as defined by Definition \ref{def_episodic}, as the optimal policy may be able to keep the pole uphold forever, and for that policy the probabilities of all finite episodes do not sum up to $1$. 
 
Importantly, while the episodic learning process as defined by Definition \ref{def_episodic} indeed encompasses the typical learning environments of all finite-horizon decision tasks, the former is \emph{not} defined as the latter. Instead, any RL process with the two \emph{intrinsic} properties prescribed by Definition \ref{def_episodic} is an episodic learning process. In the next section we will see that these two conditions are enough to guarantee nice theoretical properties that were previously typically only assumed in continual tasks.

\section{Episodic learning process is ergodic}
\label{sec_ergodic}

We first prove that episodic learning process admits unique and meaningful stationary distributions. 

\addtocounter{theorem}{1}
\begin{lemma}
	\label{lemma_irreducible}
	(Episodic learning process is irreducible.) In any episodic learning process $\M$, for every policy $\pi$, let $\SS_{\pi}$ be the set of states that are reachable under $\pi$, the induced markov chain $M_{\pi}$ admits a rollout that transits from any $s\in S_\pi$ to any $s'\in S_\pi$ in a finite number of steps.
\end{lemma}
\addtocounter{theorem}{-1}
\proofvspace
\vspace{-0.05in}
\begin{proofidea}
	Due to homogeneity, a reachable state must be reachable in one episode from any terminal state, thus any $s$ can reach any $s'$ in two episodes (through an terminal state).
	See \ref{proofsec_lemma_irreducible} for details.
\end{proofidea}

\vspace{0.07in}
\addtocounter{theorem}{1}
\begin{lemma}
	\label{lemma_recurrent}
	(Episodic learning process is positive recurrent.) In any episodic learning process $\M$, for every policy $\pi$, let $s\in \SS_\pi$ be any reachable state in the induced markov chain $M_\pi$, then $\E{\pi} [T_s] < +\infty$ (where $\E{\pi} [T_s]$ is the mean recurrence time of $s$, as defined in Section \ref{sec_preliminaries}).
\end{lemma}
\addtocounter{theorem}{-1}
\proofvspace
\vspace{-0.05in}
\begin{proofidea}
	If the first recurrence of $s$ occurs in the $n$-th episode, then the conditional expectation of $T_s$ in that case is bounded by $n \* \E{\xi\sim \pi} |\xi|$, which is a finite number due to the finiteness condition
	. More calculations show that the mean of $T_s$ after averaging over $n$ is still finite. 
	See \ref{proofsec_lemma_recurrent} for details.
\end{proofidea}

\vspace{0.07in}
\begin{theorem}
	\label{thm_stationary_dist}
	In any episodic learning process $\M$, every policy $\pi$ has a unique stationary distribution $\rho_{\pi}=\rho_{\pi} M_{\pi}$.
\end{theorem} 
\begin{proof}
	Lemma \ref{lemma_irreducible} and Lemma \ref{lemma_recurrent} show that the Markov chain $M_{\pi}$ is both irreducible and positive recurrent over $S_{\pi}$, which gives the unique stationary distribution over $S_{\pi}$ by Proposition \ref{thm_mc_stationary}. Padding zeros for unreachable states completes the row vector $\rho_{\pi}$ as desired.
\end{proof}

\vspace{-0.1in}
From Proposition \ref{def_ergodic} we also know that the steady-state probability $\rho_\pi(s)$ of state $s$ equals the reciprocal of its mean recurrence time $1/\E{\pi}[T_s]$. Theorem \ref{thm_stationary_dist} tells us that \emph{if} an episodic learning process $\M$ converges, it can only converges to a single steady state. But the theorem does not assert that $\M$ will converge at all. In general, the marginal distribution $\rho\step{t}_\pi$ could be \emph{phased} over time and never converge. 
Nevertheless, the following lemma shows that minimal diversity/randomness among the lengths of admissible episodes is enough to preserve ergodicity in episodic learning process.

\vspace{+0.06in}
\addtocounter{theorem}{1}
\begin{lemma}
	\label{lemma_ergodic}
	 In any episodic learning process $\M$, the markov chain $M_{\pi}$ induced by any policy $\pi$ converges to its stationary distribution $\rho_\pi$ if $M_\pi$ admits two episodes with co-prime lengths; that is, $\lim\limits_{t \ra \infty} \rho_\pi\step{t} = \rho_{\pi}$ if there exist different $\xi, \xi' \in \Xi_\pi$ such that $\gcd(|\xi|,|\xi'|)=1$.
\end{lemma}
\addtocounter{theorem}{-1}
\proofvspace
\vspace{-0.06in}
\begin{proofidea}
	Lemma \ref{lemma_irreducible} has shown that $M_{\pi}$ is irreducible over $\SS_{\pi}$, so we only need to identify one \emph{aperiodic} state $s\in \SS_{\pi}$, then Proposition \ref{def_ergodic} gives what we need.
	Intuitively, an episode is a recurrence from the terminal subset $\terminals$ to the subset itself, so having co-prime episode lengths entails that, if we aggregate all terminal states into one, the aggregated ``terminal state'' is an aperiodic state. The same idea can be extended to the multi-terminal-state case. See Appendix \ref{proofsec_lemma_ergodic} for the formal proof.	
\end{proofidea}

The co-primality condition in Lemma \ref{lemma_ergodic} might look restrictive at first glance, but there is a simple and general way to slightly modify \emph{any} episodic learning model $\M$ to turn it into an equivalent model $\M^+$ such that $\M^+$ does satisfy the co-primality condition. The idea is to inject a minimum level of randomness to every episode by introducing an single auxiliary state $s_{null}$ so that an episode in the modified model may start either directly from a terminal state to a normal state as usual, or go through a one-step detour to $s_{null}$ (with some probability $\epsilon$) before the ``actual'' episode begins. 

\begin{definition}
	\label{def_perturbed_model}
	An \textbf{$\epsilon$-perturbed model} of episodic learning process $\M=(\SS,\AS,\R,\T,\rho_0)$ is a MDP $\M^+ = (\SS^+, \AS, \R, \T^+, \rho_0)$, where $\SS^+ = \SS \cup \{s_{null}\}$, and with constant $0<\epsilon<1$,
	\[
		\T^+(s'|s,a) =~~
		\begin{tabular}{c|cc}
			\diagbox{$s\in$}{$s'\in$}	& $\SS$ 		 & $\{s_{null}\}$ 	\\ 
			\hline 
			$\SS_\perp$ 				& $(1-\epsilon)P(s'|s,a)$  & $\epsilon$ \\ 
			$\SS \setminus \SS_\perp$	& $\T(s'|s,a)$ 	 & $0$				\\ 
			$\{s_{null}\}$ 				& $\T(s'|s_0,a)$  & $0$				\\ 
		\end{tabular} 
	~~.\]
	
	We also prescribe $R(s_{null})=0$, $\rho_0(s_{null})=0$, $\gamma(s_{null})=1$ for the auxiliary state $s_{null}$ in $\M^+$.
\end{definition}

Note that the null state $s_{null}$ can only be reached at the beginning of an episode in a one-shot manner. Also, by the definition of terminal state (in Definition \ref{def_episodic}), $s_{null}$ is not a terminal state of $\M^+$.

\vspace{+0.08in}
\begin{theorem}
	\label{thm_limiting_dist}
	In the $\epsilon$-perturbed model $\M^+$ of any episodic learning process, every policy $\pi$ has a limiting distribution $\lim\limits_{t \ra \infty} \rho_\pi^{+ \scriptscriptstyle (t)} = \rho^+_{\pi}$, and the induced Markov chain $M^+_\pi$ is ergodic.
\end{theorem}
\vspace{-0.08in}
\proofvspace
\begin{proofidea}
	The detour to $s_{null}$ at the beginning of an admissible episode $\xi$ of length $n$ gives another admissible episode $\xi'$ of length $n+1$. As $\gcd(n+1,n)=1$ for any positive integer $n$, the two episodes $\xi$ and $\xi'$ are co-prime in length, so by Lemma \ref{lemma_ergodic}, $M^+_\pi$ is ergodic and has limiting distribution. See Appendix \ref{proofsec_thm_limiting_dist} for the rigorous proof.
\end{proofidea}

Moreover, the following theorem verifies that the perturbed episodic learning model preserves the decision-making related properties of the original model of the learning process. 

\vspace{+0.03in}
\begin{theorem}
	\label{thm_equa}
	In the $\epsilon$-perturbed model $\M^+$ of any episodic learning process $\M$, for every policy $\pi$, let $Q_\pi$ and $Q^+_\pi$ be the value functions of $\pi$ in $\M$ and $\M^+$, and let $\rho_\pi$ and $\rho^+_\pi$ be the limiting distribution under $\pi$ in $\M$ and $\M^+$, we have $Q_\pi=Q_\pi^+$ over $\SS\times\AS$, and $\rho_\pi \propto \rho^+_\pi$ over $\SS$. Specifically,
	\eq{equi_rho}{ 
	\rho_\pi(s) = \rho^+_\pi(s) / \left(  1-\rho^+_\pi(s_{null}) \right)  ~,~~ \forall s\in \SS
	.}
\end{theorem}
\vspace{-0.03in}
\proofvspace
\begin{proofidea}
	$Q^+_\pi$ remains the same because $\gamma(s_{null})=1$ and $R(s_{null})=0$, so $s_{null}$ is essentially transparent for any form of value assessment.
	Regarding $\rho^+_\pi$, we use a \emph{coupling argument} to prove the proportionality. Intuitively we can simulate the rollout of $\M^+$ by inserting $s_{null}$ into a given trajectory of $\M$. This manipulation will certainly not change the empirical frequencies for states other than $s_{null}$, so the proportionality between the \emph{empirical} distributions is obvious. But the challenge is to connect the stationary distributions \emph{behind} the empirical data -- as $\M$ may not be ergodic, there is no guaranteed equivalence between its empirical and stationary distribution (the lack of such equivalence was our original motivation to introduce $\M^+$). So instead, we exploit the coupling effects between the two coupled trajectories to establish formal connection between the underlying steady states of $\M$ and $\M^+$. See Appendix \ref{proofsec_thm_equa} for the complete proof.
\end{proofidea}

The inserted state $s_{null}$ is intended to behave transparently in decision making while mixing the phases in empirical distribution at a cost of wasting at most one time step per episode. In fact, the proof of Theorem \ref{thm_equa} suggests that we do not need to perturb the learning process in reality, but either a post-processing on given trajectories from the original model $\M$, or an on-the-fly \emph{re-numbering} to the time index of states during the rollout of $\M$, would be enough to simulate the $\epsilon$-perturbation. With this perturbation trick in mind, we can now safely claim that \emph{all} episodic learning processes that \emph{we would consider} in RL are ergodic and admit limiting distributions -- episodicity means ergodicity.

\section{Conceptual implications}
\label{sec_conceptual}

The ubiquity of steady state and ergodicity in episodic RL tasks may help connect concepts that are used in both episodic and continual RL yet were defined differently in the two RL formalisms.
 
For example, regarding performance measure, the episode-wise performance $J_{epi}$ was previously only well defined for episodic RL tasks, while the steady-state performance $J_{avg}$ was only well defined for continual RL tasks. The existence of unique stationary distribution $\rho_\pi$ in episodic learning process enables a rigorous connection between these two fundamental performance measures.
\begin{theorem}
	\label{thm_avg_epi}
	In any episodic learning process $\M$, for any policy $\pi$, $J_{epi}(\pi) = J_{avg}(\pi) \* \E{\pi}[T]$. More explicitly, 
	\vspace{-0.05in}
	\eq{avg_epi}{
		\Exp{s\sim\rho_\pi}{\R(s)} = \frac{ \Exp{\zeta\sim M_\pi}{\sum_{t=1}^{T(\zeta)} \R(s_t)} }{ \Exp{\zeta\sim M_\pi}{T(\zeta)} }
	}
	, where $T(\zeta)=\min\{t\geq 1:s_t\in\terminals\}$ is as defined in Section \ref{sec_preliminaries}. 
\end{theorem}
\proofvspace
\begin{proofidea}
	Averaging both sides of the Bellman equation (i.e. \eqref{bellman}) over distribution $\rho_\pi$, re-arranging a bit, then substituting in $\ga(s)=\indicator{s\not\in\terminals}$ will give the result. See Appendix \ref{proofsec_thm_avg_epi} for details.
\end{proofidea}

Note that the probability spaces at the two sides of \eqref{avg_epi} are different, one averaged over states, the other averaged over trajectories. 
Since in Theorem \ref{thm_avg_epi} the reward function $\R$ can be arbitrary, the theorem gives a safe way to move between the two probability spaces for \emph{any} function of state. 

In particular, for arbitrary $s^*$ substituting $R(s)=\indicator{s=s^*}$ into \eqref{avg_epi} gives $\rho_\pi(s^*)=\mu^1_\pi(s^*)$, which immediately proves that the episode-wise visitation distribution $\mu^1_\pi$ (when well defined) is equal in values to  the stationary distribution $\rho_\pi$. Interestingly, in continual RL, stationary distribution $\rho_\pi$ is used to formulate the concept of \emph{on-policy distribution} which intuitively means ``the data distribution of policy $\pi$'', but the same term ``on-policy distribution'' is (re-)defined as the episode-wise visitation distribution $\mu^1_\pi$ in episodic RL~\cite{2018:RL}. In light of their equality now, we advocate to replace $\mu_{\pi}^1$ with $\rho_\pi$ as the definition of ``on-policy distribution'' in episodic RL (see Appendix \ref{sec_on_policy} for more arguments). 

Table \ref{tab_summary} summarizes the discussions around conceptual unification so far in the paper. In this unified perspective, for both episodic and continual RL we work on the MDP model of the learning (instead of the decision) environment. On-policy distribution always means the unique stationary distribution of the policy. The performance measures remain unchanged, but they are now rigorously connected by a factor of $\E{\pi}[T]$ (i.e. the average episode length under the policy $\pi$).

\begin{table} [h!]
	\begin{center}
		\begin{tabular}{|l|l|l|l|}
			\hline \rule[-1ex]{0pt}{3ex}
			& Target Environment  	& On-policy Distribution			& Performance Measure \\ 
			\hline \rule[-1ex]{0pt}{3ex} 
			Episodic	& $\M_D$ (finite-horizon)	 	& $\propto \E{\pi}[\sum_{t\leq T} \indicator{s_t=s}]$ 	& $J_{epi}$ \\ 
			\hline \rule[-1ex]{0pt}{3ex} 
			Continual	& $\M_L$ ~(assumed ergodic)  	& $\rho_\pi$ ~( $=\lim\limits_{t \ra \infty} \rho_\pi\step{t}$ )  	& $J_{avg}$\\ 
			\hline \rule[-1ex]{0pt}{3ex} 
			Episodic (*)& $\M_L$ ~(proved ergodic) 	& $\rho_\pi$ ~( $=\lim\limits_{t \ra \infty} \rho_\pi\step{t}$ ) 	& $J_{avg}\*\E{\pi}[T]$ ~~($=J_{epi}$) \\
			\hline
		\end{tabular} 
	\end{center}
	\caption{Concept comparison between RL settings. Our account is summarized in the row with (*).}
	\label{tab_summary}
\end{table}

\section{Algorithmic implications}
\label{sec_algo}

Sample inefficiency and learning instability are two major challenges of RL practice at the moment. Being unified in the conceptual level, we can now borrow ideas from continual RL to efficiently collect stable training data for practical problems in episodic RL: Write the quantity we want to compute as an expectation over the stationary distribution of some policy $\beta$, perturb the learning environment if needed, rollout $\beta$ in the (perturbed) environment and wait for it to converge to the steady state, after which we obtain an unbiased sample point in every single time step. As a demonstration, the following subsection applies this idea to policy gradient algorithms.

\subsection{Policy gradient}
\label{sec_pg}

The availability of steady state in episodic tasks allows us to write the policy gradient of the episode-wise performance as a steady-state property of the corresponding policy.
\begin{theorem}
	\label{thm_pg}
	(Steady-State Policy Gradient Theorem.) In episodic learning process $\M$, let $Q_\pi$ be the value function \eqref{q_pi} using $\ga(s)=\indicator{s\not\in \terminals}$, then for any differentiable policy function $\pi(s,a;\theta)$,
	\eq{pg}{
		\nabla_\theta~ J_{epi}(\theta) = \Big( \E{\zeta\sim \pi(\theta)}[T]-1 \Big) \* 
		\Exp{s,a\sim \rho_{\theta}}{ Q_{\theta}(s,a) \, \nabla_\theta \log\pi(s,a;\theta) ~\Big|~ s\not\in \SS_\perp} 
	}
	where $s,a\sim \rho_{\theta}$ means $s\sim \rho_{\pi(\theta)}, a\sim \pi(s;\theta)$.
\end{theorem}
\proofvspace
\vspace{-0.01in}
\begin{proofidea}
	Applying the $\nabla_\theta$ operator over both sides of the Bellman equation \eqref{bellman}, then averaging both sides of the resulted equation over stationary distribution $\rho_\pi$, and re-arranging, will give
	\eq{pg_proof1}{
		\Exp{s,a\sim \rho_\theta}{\Big(1-\gamma(s)\Big) \nabla_\theta Q_\theta(s,a)} = 
		\Exp{s,a\sim \rho_\theta}{\gamma(s) \, Q_\theta(s,a)\nabla_\theta \log \pi(s,a;\theta)}
	.} 
	Substituting $\ga(s)=\indicator{s\not\in \terminals}$ into \eqref{pg_proof1} will leave only terminal states at LHS while leaving only non-terminal states at RHS. For terminal states, $\nabla Q_\theta(s_\perp,a) = \nabla J_{epi}$, which brings in the performance objective. Further re-arranging will give what we want. See the complete proof in Appendix \ref{proofsec_thm_pg}.
\end{proofidea}

Comparing with the classic policy gradient theorem \eqref{discounted_pg1}, Theorem \ref{thm_pg} establishes an \emph{exact equality} directly from the gradient of the \emph{end-to-end} performance $J_{epi}$ to an conditional expectation under the steady state of the target policy $\pi(\theta)$. In other words, Steady-State Policy Gradient (SSPG) algorithms have lossless surrogate objective $\tilde{J}_{SSPG} \define J_{epi}$, which can be estimated (truly) unbiasedly with estimator $F_{SSPG}(s_t,a_t,\theta) \define (\widehat{\E{\theta}[T]}-1) \* \widehat{Q_\theta} \* \nabla \log \pi(s_t,a_t;\theta)$ at each and every non-terminal step $t$ after the convergence time (thanks to the guaranteed ergodicity). We can further reduce the variance of the estimator $F_{SSPG}$ by running $K$ parallel rollouts~\cite{2016:a3c,2017:ppo}, in which case the step data at time $t$ across all these independent rollouts form an i.i.d sample of size $K$.

Comparing with A3C~\cite{2016:a3c}, one of the best ``online'' policy gradient algorithms, the SSPG estimator as presented above differs in three places. First, A3C uses the discounted value functions $Q^\gammac_\pi$ (and a $V^\gammac_\pi$ baseline) with $\gamma_c<1$ while SSPG uses $\gamma_c=1$. Interestingly, while we already know that A3C does not compute the unbiased policy gradient because of this discounting~\cite{2014:thomas}, it is still less clear what it \emph{computes instead} in episodic tasks (\cite{2014:thomas} only gave analysis in continual cases), but the proofs of Theorem \ref{thm_pg}  and \ref{thm_avg_epi} actually shed some light on the latter issue. See Appendix \ref{sec_pg_algo} for the details. 

Secondly, in the SSPG estimator there is an additional term $\widehat{\E{\theta}[T]}-1$, which is the estimated Average Episode Length (\emph{AEL} for short) of $\pi(\theta)$ minus $1$. This AEL term was often treated as a ``proportionality \emph{constant}'' and thus omitted in the literature~\cite{2018:RL}. However, Theorem \ref{thm_pg}, and specifically the subscript $\theta$ in $\E{\theta}[T]$, makes it explicit that this term is not really a constant as the policy parameter $\theta$ is certainly changing in any effective learning. The corresponding average episode length can change by \emph{orders of magnitude}, both increasingly (e.g. in Cart-Pole) and decreasingly (e.g. in Mountain-Car). While the dramatically changed AEL will not alter the gradient direction, it does affect \emph{learning rate control} which is known to be of critical importance in RL practice. 

We examined the impact of the changing AEL to the quality of policy gradient estimation in the Hopper environment of the RoboSchoolBullet benchmark~\cite{pybullet}. We implemented two variants of SSPG, one using exactly $F_{SSPG}$, the other omitting the AEL term. Both used constant learning rate. In order to observe the potential bias, we run $K=1,000,000$ parallel rollouts to reduce variance as much as possible, and used data from the single step at $\Delta t = 3 \* AEL$ across all rollouts to make each policy update. We then compared the estimated gradients from both SSPG variants with the ground-truth policy gradient 
at each update time (Figure \ref{fig:hopper_reinforce} in appendix), and observed that the exact SSPG estimator indeed well aligns with the ground-truth gradients throughout the learning process, while the variant without the AEL term exhibits significant bias in its gradient estimates after $100$ updates, due to a $7\times$ increase of AEL along the way. See \ref{sec_ael} for experimental details.

Lastly, A3C works on the raw environments without perturbation and also does not wait for steady-state convergence. The next subsection is dedicated to the ramifications in this aspect.

\subsection{Recursive perturbation}
\label{sec_perturb}

The $\epsilon$-perturbation trick introduced in Section \ref{sec_ergodic} perturbs an episode for at most one step (and only at the beginning of the episode, with arbitrarily small probability $\epsilon$). The goal there is to enforce steady-state convergence while adding only \emph{minimal} perturbation to the raw environment. In practice, there is a simple way to ``augment'' the perturbation to further facilitate \emph{rapid} convergence, so that we don't need to wait too long before enjoying the efficiency of few-step sampling under steady state.   

The key idea is that the same perturbation trick can be applied again to an already perturbed environment, which creates a detour visiting up to two null states at the beginning of an episode (and the ``doubly perturbed'' model must still remain ergodic). Repeatedly applying the same trick to the limit, we end up with a perturbed model with only one null state but the null state has a self-looping probability of $\epsilon$. The time steps that the agent will spend on the null state in an episode follows the geometric distribution with expectation $\epsilon / (1-\epsilon)$, so with $\epsilon$ approaching to $1$, the perturbation effect can be arbitrarily large. We call this augmented perturbation method, \emph{recursive perturbation}. 

\begin{figure}[h]
	\centering
	\includegraphics[width=0.9 \linewidth]{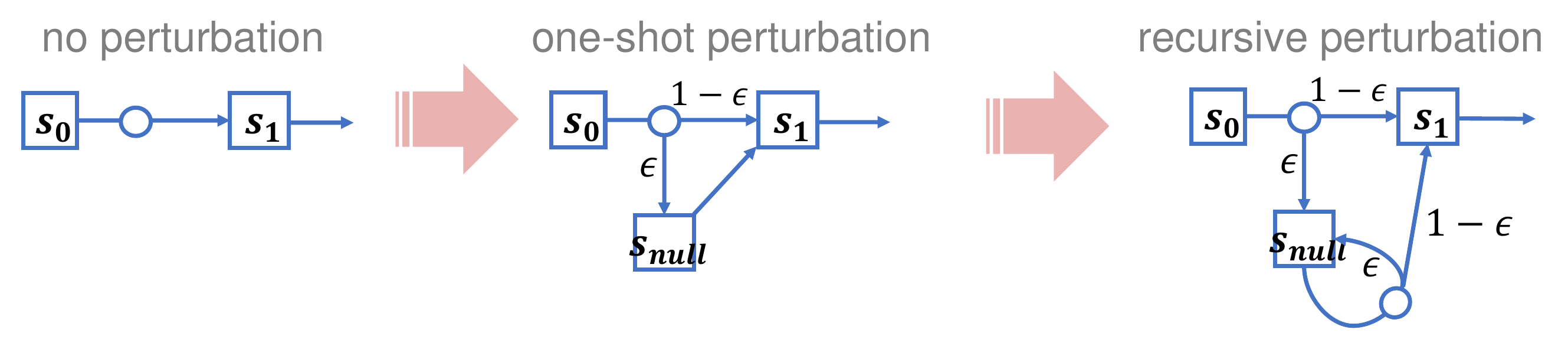}
	\caption{Illustration of the perturbation trick}
	\label{fig:perturbation}
\end{figure}

Interestingly, we empirically found that recursive perturbation with the particular $\epsilon=1-1/\E{}[T]$ appears to enforce steady-state convergence in $t^*=3\E{}[T]$ steps across \emph{various} environments. Specifically, we examined the time evolution of the marginal state distribution in three representative environments of RoboSchoolBullet~\cite{pybullet}: Hopper, Humanoid, and HalfCheetah. Without perturbation, the raw environment of Hopper indeed converges around $t = 2 \* AEL$, but that of Humanoid converges $\times 10$ slower, around $t = 24 \* AEL$. Moreover, without perturbation the time evolution in HalfCheetah does not converge at all, but instead periodically fluctuates around its steady state forever. These observations suggest that perturbation is indeed \emph{necessary} to ensure convergence in real-world tasks. On the other hand, after applying the recursive perturbation with $\epsilon=1-1/\E{}[T]$, we observed steady-state convergence at $t=3 \* AEL$ for both Humanoid and HalfCheetah.

To further test the generality of the ``$3$-AEL convergence'' phenomenon, we designed, as a worst-case scenario, a synthetic environment in which an episode goes through the state space $\SS=\{1,\dots, n\}$ in a strictly round-robin manner (regardless of the actions). The rollout of such environment is fully deterministic, and probabilities of its state distribution $\rho\step{t}$ fully concentrate at the single state $s_t = (t \mod n)$. We then applied the recursive perturbation with $\epsilon = 1-1/n$ to such environments with episode length ranging from $n=20$ to $2000$, and observed that for all of them, the marginal distribution $\rho\step{t^*}$ at $t^*=3n$ becomes \emph{uniformly} distributed after the perturbation. Details about the perturbation experiments, on both RoboSchoolBullet and the synthetic one, can be found in \ref{sec_perturb_exp}.

\section{Discussions}
\label{sec_discussion}

Ergodicity is a phenomenon inherently linked to recurrence. In episodic tasks there are recognized recurrent states, the terminal states, thus its guaranteed ergodicity should not be too surprising with hindsight. Intriguingly, it turns out that in the end what becomes more suspicious is the ergodicity presumption in the continual setting, which requires the existence of some \emph{implicit} recurrent states in the environment dynamic (see Proposition \ref{thm_mc_stationary} and \ref{def_ergodic}) -- but if that were true, wouldn't it imply that an ergodic ``continual task'' is in fact episodic with respect to these implicit recurrent states (which serve as hidden terminal states of the task)? It then follows that a truly and purely continual task \emph{cannot} be ergodic as usually postulated. Note that our definition of episodic learning process is not based on perceived category in the view of an observer, but based on intrinsic properties of the environment itself. In particular, the set of terminal states is defined as the equivalence closure of the set of initial states prescribed by $\rho_0$ (see Definition \ref{def_episodic}). It would be interesting to see if some  ``as-perceived'' continual task is actually episodic learning process under a differently specified $\rho_0$.

The methodology presented in Section \ref{sec_algo} essentially reduces the data collection problem of episodic RL to a Markov Chain Monte Carlo (MCMC) problem, so insights from the MCMC theory and practice could be potentially beneficial to RL too. In particular, more \emph{mixing time analysis} and \emph{non-equilibrium analysis} on episodic RL process can be another thread of interesting future works.

The idea of modeling the resetting procedure in episodic learning as normal steps in a multi-episode MDP has been discussed in \cite{2011:horde, 2011:insights, 2017:unifying}, which also used product-form value functions~\cite{2011:horde} to ``block'' the $\ga$ from propagating across episode boundaries. Our paper used similar ideas, but made a forward move by identifying strong \emph{structures} in the MDP thus constructed and by demonstrating both conceptual and pragmatic benefits of our theoretical insights. 
Importantly, as pointed out earlier in this paper, the choice of connecting finite episodes into an infinite-horizon MDP is not only about \emph{how} to formulate an environment, but more about \emph{which} environment has been formulated -- by looking at the episodic learning processes as defined, we are shifting the target of the formalism, intentionally or unintentionally, from the ``one-shot'' decision/testing environment to the real learning environment that the RL agent is \emph{physically} posited in (see Table \ref{tab_summary}).   

Comparing with other MDP-related research domains (e.g. \cite{1995:dpoc}), the engagement with a multi-episode MDP seems to be something \emph{unique} to model-free learning, as being able to repeat a task is an essential prerequisite to learn any task. We thus believe that explicitly and systematically studying the \emph{learning MDPs} of finite-horizon tasks is a topic that deserves more attention \emph{particularly} from the RL community. 


\section*{Broader Impacts}
This section discusses potential social impacts of this paper, as required by the NeurIPS 2020 program committee. This work mostly looked at a theoretical foundation of reinforcement learning, which, in the author's view, contributes to \emph{understand} RL as a general and \emph{natural} phenomenon of the world and of our society. 
At the engineering side, RL is in its nature an ``online'' learning paradigm that requires close interaction with the surrounding environment. This interaction, as triggered along with the learning, could possibly have unexpected effects especially if a less understood RL algorithm is deployed. The author hopes insights from this paper, on both theoretically justified and questionable aspects of existing RL practice, can contribute to more informed and responsible decisions when using RL in the real world. 
Finally, the perturbation trick and the new policy gradient estimator proposed in the paper may be integrated into fully-fledged RL algorithms in the future by others, which could in turn be used as ethically-neutral tools to influence the society, inevitably in both expected and unexpected ways.

\section*{Acknowledgments}
The author of the paper would like to thank his colleague Xu Wang for inspiring discussions on this work. The author also appreciates the many helpful comments from the anonymous reviewers of NeurIPS 2020 on earlier version of the paper. As funding disclosure, this research work received no project-dedicated financial support or engagement from any third party.


\small
\bibliographystyle{named}
\bibliography{rl}

\begin{thebibliography}{36}
\providecommand{\natexlab}[1]{#1}
\providecommand{\url}[1]{\texttt{#1}}
\expandafter\ifx\csname urlstyle\endcsname\relax
  \providecommand{\doi}[1]{doi: #1}\else
  \providecommand{\doi}{doi: \begingroup \urlstyle{rm}\Url}\fi

\bibitem[pyb()]{pybullet}
Pybullet gymperium.
\newblock URL \url{https://github.com/benelot/pybullet-gym}.

\bibitem[Bertsekas(1995)]{1995:dpoc}
D.~P. Bertsekas.
\newblock \emph{Dynamic programming and optimal control}, volume~1.
\newblock Athena scientific Belmont, MA, 1995.

\bibitem[Brockman et~al.(2016)Brockman, Cheung, Pettersson, Schneider,
  Schulman, Tang, and Zaremba]{2016:gym}
G.~Brockman, V.~Cheung, L.~Pettersson, J.~Schneider, J.~Schulman, J.~Tang, and
  W.~Zaremba.
\newblock Openai gym.
\newblock \emph{arXiv preprint arXiv:1606.01540}, 2016.

\bibitem[Degris et~al.(2012)Degris, White, and Sutton]{2012:opac}
T.~Degris, M.~White, and R.~S. Sutton.
\newblock Off-policy actor-critic.
\newblock \emph{arXiv preprint arXiv:1205.4839}, 2012.

\bibitem[Duan et~al.(2016)Duan, Chen, Houthooft, Schulman, and
  Abbeel]{2016:benchmarking}
Y.~Duan, X.~Chen, R.~Houthooft, J.~Schulman, and P.~Abbeel.
\newblock Benchmarking deep reinforcement learning for continuous control.
\newblock In \emph{International Conference on Machine Learning}, pages
  1329--1338, 2016.

\bibitem[Fujimoto et~al.(2018)Fujimoto, Van~Hoof, and Meger]{2018:td3}
S.~Fujimoto, H.~Van~Hoof, and D.~Meger.
\newblock Addressing function approximation error in actor-critic methods.
\newblock \emph{arXiv preprint arXiv:1802.09477}, 2018.

\bibitem[Gu et~al.(2016)Gu, Lillicrap, Sutskever, and Levine]{2016:naf}
S.~Gu, T.~Lillicrap, I.~Sutskever, and S.~Levine.
\newblock Continuous deep q-learning with model-based acceleration.
\newblock In \emph{International Conference on Machine Learning}, pages
  2829--2838, 2016.

\bibitem[Haarnoja et~al.(2018)Haarnoja, Zhou, Abbeel, and Levine]{2018:sac}
T.~Haarnoja, A.~Zhou, P.~Abbeel, and S.~Levine.
\newblock Soft actor-critic: Off-policy maximum entropy deep reinforcement
  learning with a stochastic actor.
\newblock In \emph{International Conference on Machine Learning}, pages
  1861--1870, 2018.

\bibitem[Kakade(2002)]{2002:natural_pg}
S.~M. Kakade.
\newblock A natural policy gradient.
\newblock In \emph{Advances in neural information processing systems}, pages
  1531--1538, 2002.

\bibitem[Lillicrap et~al.(2015)Lillicrap, Hunt, Pritzel, Heess, Erez, Tassa,
  Silver, and Wierstra]{2015:ddpg}
T.~P. Lillicrap, J.~J. Hunt, A.~Pritzel, N.~Heess, T.~Erez, Y.~Tassa,
  D.~Silver, and D.~Wierstra.
\newblock Continuous control with deep reinforcement learning.
\newblock \emph{arXiv preprint arXiv:1509.02971}, 2015.

\bibitem[Marbach and Tsitsiklis(2001)]{2001:continual_pg}
P.~Marbach and J.~N. Tsitsiklis.
\newblock Simulation-based optimization of markov reward processes.
\newblock \emph{IEEE Transactions on Automatic Control}, 46\penalty0
  (2):\penalty0 191--209, 2001.

\bibitem[Meyn and Tweedie(2012)]{2012:markov_chain}
S.~P. Meyn and R.~L. Tweedie.
\newblock \emph{Markov chains and stochastic stability}.
\newblock Springer Science \& Business Media, 2012.

\bibitem[Mnih et~al.(2015)Mnih, Kavukcuoglu, Silver, Rusu, Veness, Bellemare,
  Graves, Riedmiller, Fidjeland, Ostrovski, et~al.]{2015:dqn}
V.~Mnih, K.~Kavukcuoglu, D.~Silver, A.~A. Rusu, J.~Veness, M.~G. Bellemare,
  A.~Graves, M.~Riedmiller, A.~K. Fidjeland, G.~Ostrovski, et~al.
\newblock Human-level control through deep reinforcement learning.
\newblock \emph{Nature}, 518\penalty0 (7540):\penalty0 529--533, 2015.

\bibitem[Mnih et~al.(2016)Mnih, Badia, Mirza, Graves, Lillicrap, Harley,
  Silver, and Kavukcuoglu]{2016:a3c}
V.~Mnih, A.~P. Badia, M.~Mirza, A.~Graves, T.~Lillicrap, T.~Harley, D.~Silver,
  and K.~Kavukcuoglu.
\newblock Asynchronous methods for deep reinforcement learning.
\newblock In \emph{International conference on machine learning}, pages
  1928--1937, 2016.

\bibitem[Pardo et~al.(2017)Pardo, Tavakoli, Levdik, and
  Kormushev]{2017:time_limit}
F.~Pardo, A.~Tavakoli, V.~Levdik, and P.~Kormushev.
\newblock Time limits in reinforcement learning.
\newblock \emph{arXiv preprint arXiv:1712.00378}, 2017.

\bibitem[Schulman et~al.(2015)Schulman, Levine, Abbeel, Jordan, and
  Moritz]{2015:trpo}
J.~Schulman, S.~Levine, P.~Abbeel, M.~Jordan, and P.~Moritz.
\newblock Trust region policy optimization.
\newblock In \emph{International Conference on Machine Learning}, pages
  1889--1897, 2015.

\bibitem[Schulman et~al.(2017{\natexlab{a}})Schulman, Chen, and
  Abbeel]{2017:equivalence}
J.~Schulman, X.~Chen, and P.~Abbeel.
\newblock Equivalence between policy gradients and soft q-learning.
\newblock \emph{arXiv preprint arXiv:1704.06440}, 2017{\natexlab{a}}.

\bibitem[Schulman et~al.(2017{\natexlab{b}})Schulman, Wolski, Dhariwal,
  Radford, and Klimov]{2017:ppo}
J.~Schulman, F.~Wolski, P.~Dhariwal, A.~Radford, and O.~Klimov.
\newblock Proximal policy optimization algorithms.
\newblock \emph{arXiv preprint arXiv:1707.06347}, 2017{\natexlab{b}}.

\bibitem[Serfozo(2009)]{2009:markov_chain}
R.~Serfozo.
\newblock \emph{Basics of applied stochastic processes}.
\newblock Springer Science \& Business Media, 2009.

\bibitem[Silver et~al.(2014)Silver, Lever, Heess, Degris, Wierstra, and
  Riedmiller]{2014:dpg}
D.~Silver, G.~Lever, N.~Heess, T.~Degris, D.~Wierstra, and M.~Riedmiller.
\newblock Deterministic policy gradient algorithms.
\newblock In \emph{ICML}, 2014.

\bibitem[Silver et~al.(2017)Silver, van Hasselt, Hessel, Schaul, Guez, Harley,
  Dulac-Arnold, Reichert, Rabinowitz, Barreto, et~al.]{2017:predictron}
D.~Silver, H.~van Hasselt, M.~Hessel, T.~Schaul, A.~Guez, T.~Harley,
  G.~Dulac-Arnold, D.~Reichert, N.~Rabinowitz, A.~Barreto, et~al.
\newblock The predictron: End-to-end learning and planning.
\newblock In \emph{Proceedings of the 34th International Conference on Machine
  Learning-Volume 70}, pages 3191--3199. JMLR. org, 2017.

\bibitem[Sutton et~al.(2014)Sutton, Mahmood, Precup, and Hasselt]{2014:qlambda}
R.~Sutton, A.~R. Mahmood, D.~Precup, and H.~Hasselt.
\newblock A new q (lambda) with interim forward view and monte carlo
  equivalence.
\newblock In \emph{International Conference on Machine Learning}, pages
  568--576, 2014.

\bibitem[Sutton(1995)]{1995:td_models}
R.~S. Sutton.
\newblock Td models: modeling the world at a mixture of time scales.
\newblock In \emph{Proceedings of the Twelfth International Conference on
  International Conference on Machine Learning}, pages 531--539. Morgan
  Kaufmann Publishers Inc., 1995.

\bibitem[Sutton and Barto(2018)]{2018:RL}
R.~S. Sutton and A.~G. Barto.
\newblock \emph{Reinforcement learning: An introduction}.
\newblock MIT press, 2018.

\bibitem[Sutton et~al.(2000)Sutton, McAllester, Singh, and Mansour]{2000:pg}
R.~S. Sutton, D.~A. McAllester, S.~P. Singh, and Y.~Mansour.
\newblock Policy gradient methods for reinforcement learning with function
  approximation.
\newblock In \emph{Advances in neural information processing systems}, pages
  1057--1063, 2000.

\bibitem[Sutton et~al.(2011)Sutton, Modayil, Delp, Degris, Pilarski, White, and
  Precup]{2011:horde}
R.~S. Sutton, J.~Modayil, M.~Delp, T.~Degris, P.~M. Pilarski, A.~White, and
  D.~Precup.
\newblock Horde: a scalable real-time architecture for learning knowledge from
  unsupervised sensorimotor interaction.
\newblock In \emph{The 10th International Conference on Autonomous Agents and
  Multiagent Systems-Volume 2}, pages 761--768, 2011.

\bibitem[Sutton et~al.(2016)Sutton, Mahmood, and White]{2016:emphatic}
R.~S. Sutton, A.~R. Mahmood, and M.~White.
\newblock An emphatic approach to the problem of off-policy temporal-difference
  learning.
\newblock \emph{The Journal of Machine Learning Research}, 17\penalty0
  (1):\penalty0 2603--2631, 2016.

\bibitem[Thomas(2014)]{2014:thomas}
P.~Thomas.
\newblock Bias in natural actor-critic algorithms.
\newblock In \emph{International conference on machine learning}, pages
  441--448, 2014.

\bibitem[Tsitsiklis and Van~Roy(2002)]{2002:tsitsiklis}
J.~N. Tsitsiklis and B.~Van~Roy.
\newblock On average versus discounted reward temporal-difference learning.
\newblock \emph{Machine Learning}, 49\penalty0 (2-3):\penalty0 179--191, 2002.

\bibitem[Van~Hasselt et~al.(2016)Van~Hasselt, Guez, and Silver]{2016:ddqn}
H.~Van~Hasselt, A.~Guez, and D.~Silver.
\newblock Deep reinforcement learning with double q-learning.
\newblock In \emph{AAAI}, volume~2, page~5. Phoenix, AZ, 2016.

\bibitem[van Hasselt(2011)]{2011:insights}
H.~P. van Hasselt.
\newblock \emph{Insights in reinforcement rearning: formal analysis and
  empirical evaluation of temporal-difference learning algorithms}.
\newblock Utrecht University, 2011.

\bibitem[Watkins(1989)]{1989:qlearning}
C.~Watkins.
\newblock Learning from delayed rewards.
\newblock \emph{PhD thesis, King's College, University of Cambridge}, 1989.

\bibitem[White(2017)]{2017:unifying}
M.~White.
\newblock Unifying task specification in reinforcement learning.
\newblock In \emph{Proceedings of the 34th International Conference on Machine
  Learning-Volume 70}, pages 3742--3750. JMLR. org, 2017.

\bibitem[Williams(1992)]{1992:reinforce}
R.~J. Williams.
\newblock Simple statistical gradient-following algorithms for connectionist
  reinforcement learning.
\newblock \emph{Machine learning}, 8\penalty0 (3-4):\penalty0 229--256, 1992.

\bibitem[Yu(2015)]{2015:YU}
H.~Yu.
\newblock On convergence of emphatic temporal-difference learning.
\newblock In \emph{Conference on Learning Theory}, pages 1724--1751, 2015.

\bibitem[Zhang et~al.(2019)Zhang, Boehmer, and Whiteson]{2019:generalized}
S.~Zhang, W.~Boehmer, and S.~Whiteson.
\newblock Generalized off-policy actor-critic.
\newblock In \emph{Advances in Neural Information Processing Systems}, pages
  1999--2009, 2019.

\end{thebibliography}

\newpage
\appendix
\appendixpage
\addappheadtotoc

\section{Notes on the MDP formulation}
\label{sec_mdp}

This section provides more discussions about the MDP formulation as introduced in Section \ref{sec_preliminaries}. A faithful implementation of this formulation was used to generate all the results reported in Section \ref{sec_algo} and Section \ref{sec_exp}.

Our MDP formulation assumes both $\SS$ and $\AS$ are countable sets. This is mainly for aligning to the standard markov chain theory, and also for enabling convenient notations like transition matrix $M$ and summation $\sum_{s,a}$. Results in this paper are readily generalizable to uncountable action spaces (which still induce countable transitions after marginalizing over the actions), and may also be generalized to uncountable state spaces based on the theory of \emph{general-state-space} markov chains~\cite{2012:markov_chain}.

Our MDP formulation also assumes state-based deterministic reward function $\R$. While this formulation was used in many previous works~\cite{2015:trpo}, some literature~\cite{2018:RL} explicitly assign stochastic rewards to transitions, in the form of $\R(r|s,a,s')=\Pr{r_{t+1}=r | s_t=s,a_t=a, s_{t+1}=s'}$. Our reward-function formulation has no loss of generality here, as one can think of a ``state'' in our formulation as a $(s,r)$ pair in stochastic-reward models, with the transition-dependent stochasticity of the reward captured by the transition function $\T$.

Our MDP formulation has replaced the (often included) discounting constant $\ga_c$ with the (often excluded) initial state distribution $\rho_0$. Similar perspective to ``downgrade'' the discounting constant was discussed in \cite{2018:RL} (Chapter 10.4). As will become evident later in the paper, the discounting constant is neither necessary for the purpose of defining the problem, nor is it necessary for treating the problem. On the other hand, an explicit specification of $\rho_0$ is necessary to define the episode-wise performance $J_{epi}$ as used in all episodic tasks, and will be also needed to define the terminal states in the formulation of episodic learning process proposed in this paper. 

A finite-horizon task can have degenerated steady states if it is formulated as an infinite-horizon MDP with an imaginary absorbing state $\absorb$. In the absorbing MDP formulation, any finite episode will end up with moving from its terminal state to the absorbing state $\absorb$, from there the rollout is trapped in the absorbing state forever without generating effective reward~\cite{2018:RL,2014:thomas}. As a result, the stationary (and limiting) distribution of the absorbing MDP of any finite-horizon task concentrates fully and only to the absorbing state, making it of limited use for designing and analyzing RL algorithms.   

The definition of the steady-state performance measure $J_{avg}(\pi) \define \lim\limits_{ T \ra \infty} \frac{1}{T} \sum_{t=1}^{T} \R(s_t) = \E{s\sim \rho_\pi}[\R(s)]$, as introduced in Section \ref{sec_preliminaries}, is based on the following \emph{ergodic theorem} of markov chain:
\begin{proposition}
	\label{thm_ergodic}
	In ergodic markov chain $M$ with stationary distribution $\rho_{\scriptscriptstyle M}$, let $f(s)$ be any function such that $\E{s\sim \rho_{\scriptscriptstyle M}}[~|f(s)|~]<\infty$, then the time-average of $f$ converges almost surely to the state-average of $f$, i.e., $\lim\limits_{T \ra \infty} \Pr{~\frac{1}{T}\sum_{t=1}^T f(s_t) = \E{s\sim\rho_{\scriptscriptstyle M}}[f(s)]~}=1$.
	(\cite{2009:markov_chain}, Theorem 74)
\end{proposition}

The value function $Q$ is often more specifically called the \emph{action-value function}. We called the $Q$-function just value function because of its symmetric role with policy function $\pi$. Moreover, in classic literature, the term ``action-value function'' was often referred specifically to 
$
Q^\gammac_\pi(s,a) \define \E{s_1\sim P(s,a), \{s_{\geq 2}\}\sim M_\pi} [\sum_{t = 1}^\infty R(s_{t}) \* \gamma_c^{t-1}]
$
, in which the constant $0\leq \gamma_c <1$ is called the \emph{discounting factor}. The specialized definition of $Q_\pi^\gammac$ entails the specialized version of Bellman equation 
$
Q_\pi^\gammac(s,a) = \E{s'\sim P(s,a), a'\sim \pi(s')} [R(s')+\gamma_c \* Q_\pi^\gammac(s',a')]
,$
and induces the specialized version of state-value function 
$
V_\pi^\gammac(s) \define \E{a\sim \pi(s)}[Q_\pi^\gammac(s,a)]
.$

The product of $\ga$ in the general-form value functions \eqref{q_pi} was originally proposed as a virtual probabilistic termination in the agent's mind~\cite{1995:td_models}, 
but was latter found useful to account for the real episode-boundaries when the MDP is used to model multi-episode rollouts \cite{2011:insights,2017:unifying}. 
Our paper uses product-form value functions for the same purpose as in \cite{2011:insights} (Section 2.1.1) and \cite{2017:unifying}. However, both \cite{2011:insights} and \cite{2017:unifying} used transition-based $\gamma$-discounting, while the value function in our treatment uses state-based discounting and is probably closer to the method of Figure 1(c) in \cite{2017:unifying}, which was considered an sub-optimal design in that paper. In fact, \cite{2017:unifying} attributes much of its main technical contribution to the adoption and contrastive analysis of transition-based discounting (against state-based discounting) which concludes that ``\emph{transition-based discounting is necessary to enable the unified specification of episodic and continuing tasks}'' (\cite{2017:unifying} Section 2.1, see also Section 6 and B). Despite this major technical disparity, we nevertheless think the approach as described in Section \ref{sec_elp} of our paper actually aligns with \cite{2011:insights} (Section 2.1.1) and \cite{2017:unifying} in terms of the bigger idea that connecting finite episodes into a single infinite-horizon MDP greatly helps with unifying the episodic and continual formalisms.


\section{Gradient estimation in RL algorithms}
\label{sec_rl}

As mentioned in Section \ref{sec_preliminaries}, many RL algorithms can be interpreted under a common and basic idea that seeks to find a parameterized policy function $\pi(\theta)$, a surrogate objective $\tilde{J}$, and an estimator function $F$, such that good policies with respect to the end-to-end performance measure (typically $J_{avg}$ in continual tasks and $J_{epi}$ in episodic tasks) can be identified by optimizing the surrogate objective $\tilde{J}$ via stochastic gradient methods, where $\nabla_\theta \tilde{J}(\theta_t)$ is estimated by averaging $F$ over some data set $\D_t$, that is
\eq{}{
\nabla_\theta \tilde{J}(\theta_t) \approx \frac{1}{|\mathcal{D}_t|}\sum_{(s,a)\in \mathcal{D}_t} F(s,a,\theta_t)
.}
RL algorithms (policy-based or value-based, on-policy or off-policy) under this framework differ ``only'' in how its policy function $\pi$ is parameterized, and in how $\tilde{J}$, $F$, and $\D$ are constructed. This section briefly reviews the specific forms of $\tilde{J}$ and $F$, as well as the prescribed distribution behind $\D_t$, that are used in some popular RL algorithms at the moment. To simplify notations we write $\rho_\theta$, $Q_\theta$, $V_\theta$, $J(\theta)$ for $\rho_{\pi(\theta)}$, $Q_{\pi(\theta)}$, $V_{\pi(\theta)}$, $J(\pi(\theta))$, respectively, from now on. As all algorithms discussed in this section use constant discounting, we write $\gamma$ (instead of $\ga_c$) for the discounting constant (only) in this section. 

In the \emph{policy-based approach}\cite{1992:reinforce, 2000:pg, 2002:natural_pg, 2014:dpg, 2015:trpo, 2017:ppo}, the policy $\pi$ is a function directly differentiable to policy parameter $\theta$ (that is, $\nabla_\theta \pi$ is computable). In its most classic and popular form, the surrogate objective $\tilde{J}$ is chosen to be a \emph{discounted episodic return} $J_{epi}^{\ga}(\theta) \define \E{s_0\sim \rho_0}[V^{\gamma}_\theta(s_0)]$, whose estimator $F$ is derived as~\cite{2000:pg}
\eqa{}{
	\nabla J_{epi}(\theta) \approx \nabla J_{epi}^{\ga}(\theta) 
	&= \sum_{t=0}^\infty \ga^t \* \E{s_t\sim \rho^{(t)}_\theta} \Exp{a_t\sim\pi(s_t;\theta)}{\nabla \log \pi(s_t,a_t;\theta) \* Q^{\ga}_\theta(s_t,a_t)} \\
	&= \frac{1}{1-\ga} \* \E{s\sim \mu^{\ga}_\theta} \Exp{a\sim\pi(s;\theta)}{\nabla \log \pi(s_t,a_t;\theta) \* Q^{\ga}_\theta(s,a)} \label{dvd_pg} \\
	&\approx \frac{1}{|\mathcal{D}|}\sum_{(s_i,a_i)\in \mathcal{D}} \nabla \log \pi(s_t,a_t;\theta) \* \widehat{Q^{\ga}_\theta}(s_i,a_i) \label{appendix_pg}
	.}
In \eqref{dvd_pg}, $\mu^{\ga}_\theta \define \sum_{t=0}^\infty \rho^{(t)}_\theta \* \ga^t (1-\ga)$ is sometimes called the \emph{(normalized) discounted visitation distribution}, and $\widehat{Q^{\ga}_\theta}(s_i,a_i)$ is some approximation of the value $Q^{\ga}_\theta(s_i,a_i)$. In above, $F_{pg}(s,a,\theta)=\nabla \log \pi(s,a;\theta) \* \widehat{Q^{\ga}_\theta}(s,a)$. 

The vanilla REINFORCE algorithm~\cite{1992:reinforce, 2018:RL} uses one-step data $(s_t,a_t)$ from a single rollout to construct the data set $\D$ for each policy update. Modern variants of it employ batch-mode updates~\cite{2015:trpo, 2016:benchmarking}, using data accumulated from multiple episodes to construct the data set $\D$. The A3C algorithm~\cite{2016:a3c} uses the same surrogate objective $J^\ga_{epi}$ and the same estimator $F_{pg}$, but constructs $\D$ using data from a small time window (e.g. five consecutive steps\cite{2016:a3c}) of multiple \emph{parallel and independent} rollouts. The PPO algorithm~\cite{2017:ppo} collects data set $\D$ in similar way, but conduct multiple policy updates on a single data set $\D$, thus improving sample efficiency. To keep the policy updates well directed, PPO uses a slightly different surrogate objective that majorizes $J^\ga_{epi}$ around the base parameter $\theta^{old}$, an idea first employed in the TRPO algorithm~\cite{2015:trpo}. In all these RL algorithms, the data set $\D$ follows the on-policy distribution of the \emph{target policy} $\pi(\theta^{old})$, and are thus called \emph{on-policy algorithms}.    

In the \emph{value-based approach}~\cite{1989:qlearning,2015:dqn,2016:ddqn,2016:naf,2015:ddpg,2018:td3,2018:sac}, the agent policy $\pi$ is parameterized indirectly \emph{through} a differentiable function $Q$. For example, $\pi$ may be a greedy policy that has zero selection-probability for all sub-optimal actions with respect to $Q(\theta)$, i.e., with $\pi(s,a;\theta)=0$ for $a\not\in \arg\max_{\bar{a}} Q(s,\bar{a};\theta)$. In this case $\nabla_\theta \pi$ is generally not computable, but $\nabla_\theta Q$ is. In the most classic form of this approach, the surrogate objective $\tilde{J}$ is chosen to be the so-called \emph{Projected Bellman Error} $J_{PBE}(\theta) \define \sum_{s\in\SS}\sum_{a\in\AS} \delta^2(s,a;\theta)$, where $\delta(s,a;\theta) \define Q(s,a;\theta) - \E{s'\sim P(s,a), a'\sim \pi(s';\theta^{old})} [R(s')+\ga Q(s',a';\theta^{old})]$, whose estimator is derived as~\cite{2015:dqn}
\eqa{}{
	\nabla J_{PBE}(\theta) 
	&\approx \E{s,a\sim \U(\SS\times\AS)} [\nabla Q(s,a;\theta) \* \delta(s,a;\theta)] \label{appendix_ql1} \\
	&\approx \frac{1}{|\mathcal{D}|}\sum_{(s_i,a_i)\in \mathcal{D}} \nabla Q(s_i,a_i;\theta) \* \widehat{\delta}(s_i,a_i;\theta) \label{appendix_ql}
	,}
where $\U(\SS\times\AS)$ can be any positive distribution over the states and actions, $\widehat{\delta}(s_i,a_i;\theta)$ is some approximation of $\delta(s_i,a_i;\theta)$, and $F_{PBE}(s,a,\theta)= \nabla Q(s,a;\theta) \* \widehat{\delta}(s,a;\theta)$. 

Similar to the case in policy-based approach, early value-based algorithms such as Q-Learning~\cite{1989:qlearning} used one-step data $(s_t,a_t)$ from a single rollout to construct the data set $\D$ for each policy update based on \eqref{appendix_ql}, while modern variants of it typically conduct batch-mode updates, again either using multiple-episode data from a single rollout~\cite{2015:dqn} or using data of small time window from parallel rollouts~\cite{2016:a3c}. The basic surrogate objective $J_{PBE}$ and its estimator $F_{PBE}$ used in \eqref{appendix_ql} can also be improved in many ways, such as using two (weakly- or un-correlated) base parameters $\theta^{old}$ in the $\delta$ function~\cite{2016:ddqn, 2018:td3}, and adding entropy-regularization terms~\cite{2017:equivalence, 2018:sac}. Variants of \eqref{appendix_ql} that are applicable to continuous action spaces were also proposed~\cite{2015:ddpg,2016:naf}. 

In order to comprehensively approximate the positive distribution $\U(\SS\times \AS)$ in \eqref{appendix_ql1}, these value-based algorithms typically employ some behavior policy $\beta$ that is more exploratory than the target policy $\pi$, so that the data set $\D$ in \eqref{appendix_ql} follows the on-policy distribution $\rho_\beta$. To improve sample efficiency and reduce auto-correlation, the sample set $\D$ used by modern value-based algorithms is usually a mixture of data from multiple behavior policies $\beta_t$'s that the agent has been using over time $t$~\cite{2015:dqn}, in which case $\D$ follows a mixture of the on-policy distributions $\rho_{\beta_t}$. As the data set $\D$ does not follow the on-policy distribution of the target policy $\pi$ (or of any single policy) in this case, algorithms based on such data set $\D$ is called \emph{off-policy algorithms}. 

As we can see, all the RL algorithms discussed above, policy-based or value-based, on-policy or off-policy, they all rely on the capability to obtain high-quality data from rollouts (sequential or parallel) that follows a desired distribution. Our work mainly concern about the theoretical underpins and proper sampling strategies to generate the data as required, and is thus complementary to most of the works reviewed above.

\section{Complete Proofs}
\subsection{Lemma \ref{lemma_irreducible}}
\label{proofsec_lemma_irreducible}
\begin{proof}
	
	Due to the homogeneity condition of episodic learning, a reachable state must be reachable in one episode and from any terminal state. So let $\xi$ be such an admissible episode trajectory under $\pi$ that go through the state $s$, and suppose $\xi$ terminates at an arbitrary terminal state $s^*$, there must be an admissible episode trajectory $\xi'$ that starts from $s^*$ and go through the state $s'$. Due to the finiteness condition of episodic learning, both $\xi$ and $\xi'$ take finite steps, so the concatenated trajectory $\xi\xi'$ contains a finite path $s\to s'$ as subsequence.
\end{proof}

\subsection{Lemma \ref{lemma_recurrent}}
\label{proofsec_lemma_recurrent}
\begin{proof}
	For the purpose of the proof, suppose we roll out $M_\pi$ starting from an arbitrary $s$, so $T_s$ is the time index of the first reccurrence of $s$ in such a rollout. 
	To make the proof rigorous, we first slightly re-formulate the probability space of $\mathbb{E} [T_s]$: Imagine again we rollout $M_{\pi}$ starting from $s$ but we terminate the rollout immediately after the rollout returns to $s$ for the first time. The sample space of such truncated rollout, denoted by $\Omega_{s}$, consists of all the finite trajectories $\xi = (s,\dots,s)$ in which $s$ shows up \emph{only} at the first and the last time step. Let $T_s(\xi)$ denote the recurrence time of $s$ (i.e. the last time step) in a specific $\xi \in \Omega_s$, the probability to obtain such a $\xi$ from the truncated rollout is $\P_{\xi}[\xi] = \prod_{t\geq 1} M_{\pi}(s_{t-1},s_t)$, and $\sum_{\xi \in \Omega_s} \P_{\xi}[\xi] = 1$ due to the finiteness condition of episodic learning process. The expected recurrence time in the truncated rollout is the same as the one in the full rollout as in the former case we only truncate \emph{after} the recurrence, so $\E{\zeta\sim M_\pi}[T_s]=\E{\xi\in \Omega_s}[T_s]$. In the rest of this proof, when we write $\E{}[T_s]$ we mean $\E{\xi\in \Omega_s}[T_s]$.
	
	Let $n_s$ be the number of episodes \emph{completed} before time $T_s$ (i.e. by time step $T_s-1$), we have $\E{} [T_s] = \sum_{k\geq 0} \Pr{n_s=k} \* \E{} [T_s | n_s=k]$. 
	
	$n_s=k$ means that the first recurrence of $s$ occurs in the $(n_s+1)$-th episode. Due to the homogeneity condition, there is a \emph{uniform} episode-wise hitting probability $\P_{\xi}[s\in \xi]$ which applies to all the $n_s+1$ episodes. Denoting $\alpha_s = \P_{\xi}[s\in \xi]$, we have $\Pr{n_s=k} = (1-\alpha_s)^{k} \alpha_s$.  
	
	On the other hand, as the recurrence time $T_s$ falls in the $n_s+1$-th episode, it must be upper bounded by the sum of lengths of all the $n_s+1$ episodes. Thus, due to the finiteness condition, we have $\E{} [T_s | n_s=k] \leq (k+1) \cdot \E{\xi}|\xi|<+\infty$. 
	
	Putting things together, we have
	$
	\E{} [T_s] \leq \sum_{k\geq 0} (1-\alpha_s)^k \alpha_s  \* (k+1) \E{\xi}|\xi| = \alpha_s \E{\xi}|\xi| \* \sum_{k\geq 0} (1-\alpha_s)^k (k+1)  = \alpha_s \E{\xi}|\xi| \* \frac{1}{\alpha_s^2} = \E{\xi}|\xi| /\alpha_s
	$. Note that $\alpha_s>0$ because $s$ is reachable in $M_\pi$.
\end{proof}

\subsection{Lemma \ref{lemma_ergodic}}
\label{proofsec_lemma_ergodic}
\begin{proof}
	Lemma \ref{lemma_irreducible} shows that $M_{\pi}$ is irreducible over $\SS_{\pi}$, so we only need to identify one \emph{aperiodic} state $s\in \SS_{\pi}$, which will prove that $M_{\pi}$ is ergodic, then by Proposition \ref{def_ergodic} the stationary distribution $\rho_{\pi}$ is also the limiting distribution over $\SS_\pi$ (and clearly $\lim\limits_{t \ra \infty} \rho_\pi\step{t} = 0 = \rho_{\pi}$ for unreachable $s\not\in \SS_\pi$).
	
	Consider episodes in $M_{\pi}$ that end at some terminal state $s_1 \in \SS_\perp$ after $n$ steps. Such an episode can start from any terminal state, including $s_1$ itself. Let $\xi_{1,1}$ be the trajectory of such an episode, which is thus a $n$-step recurrence of state $s_1$. On the other hand, due to the assumed condition in the lemma, for some such $n$ we can find episodes with co-prime length $m$ with $\gcd(n,m)=1$. Let $\xi_{1,2}$ be such an episode trajectory of length $m$, which starts again at $s_1$ but ends at some terminal state $s_2\in \SS_\perp$.
	
	Now if $s_1=s_2$, then $\xi_{1,2}$ is another recurrence trajectory of $s_1$ which has co-prime length with $\xi_{1,1}$, thus $s_1$ is aperiodic. Otherwise if $s_1$ and $s_2$ are different terminal states, then we replace the initial state of $\xi_{1,1}$ from $s_1$ to $s_2$, obtaining a third episode trajectory $\xi_{2,1}$, which starts from $s_2$ and ends at $s_1$ after $|\xi_{2,1}|=|\xi_{1,1}|=n$ steps. Consider the concatenated trajectory $\xi_{1,2}\xi_{2,1}$, which is the trajectory of two consecutive episodes that first goes from $s_1$ to $s_2$ in $m$ steps, then goes from $s_2$ back to $s_1$ in $n$ steps, thus form a $(m+n)$-step recurrence of $s_1$. As $\gcd(m+n,n)=\gcd(m,n)=1$ for any $n\neq m$, we know $\xi_{1,1}$ and $\xi_{1,2}\xi_{2,1}$ are two recurrences of $s_1$ with co-prime lengths, thus $s_1$ is still aperiodic.  
\end{proof}

\subsection{Theorem \ref{thm_limiting_dist}}
\label{proofsec_thm_limiting_dist}
\begin{proof}
	Consider an arbitrary admissible episode $\xi$ in the Markov chain $M_\pi$ induced by $\pi$ in the original learning model $\M$. Let the episode length $|\xi|=n$ be an arbitrary integer $n>0$. $\xi$ is still a possible episode under $\pi$ in the perturbed model $\M^+$. In particular, from the definition of $P^+$ (in Definition \ref{def_perturbed_model}) we have $\P_{\xi\sim M^+_\pi}[\xi] = (1-\epsilon) \* \P_{\xi\sim M_\pi}[\xi]$.
	
	On the other hand, in the perturbed model, the trajectory $\xi=(s_0,s_1,\dots,s_n)$ is accompanied with a ``detoured'' trajectory $\xi'=(s_0,s_{null},s_1,\dots,s_n)$, which is the same with $\xi$ except for the detour steps $s_0\ra s_{null} \ra s_1$. The detour is always possible under any policy $\pi$ as both $P^+(s_{null}|s_0,a)=\epsilon$ and $P^+(s_1|s_{null},a)=P(s_1|s_0,a)$ are action-agnostic and non-zero. So, $M_{\pi}$ admits both $\xi$ and $\xi'$, which have episode lengths $n$ and $n+1$, respectively. As $\gcd(n+1,n)=1$ for any positive integer $n$, we got two episodes with co-prime lengths now, then from Lemma \ref{lemma_ergodic} we know $M^+_\pi$ is ergodic and thus has limiting distribution.
\end{proof}

\subsection{Theorem \ref{thm_equa}}
\label{proofsec_thm_equa}
\begin{proof}
	As terminal states have $\gamma=0$, the calculation of the $\prod \ga_\tau$ term in $Q^+_\pi(s,a)$ (in Eq. \eqref{q_pi}) truncates at the end of an episode. As the auxiliary state $s_{null}$ is reachable only from a terminal state in the very first step of an episode, for any other $s\not\in \SS_\perp$ the whole transition model remains the same within an episode, thus $Q_\pi^+(s,a)=Q_\pi(s,a)$ for $s\not\in \SS_\perp$. For $s\in \SS_\perp$, their action values is also unchanged because, with $\gamma(s_{null})=1$ and $R(s_{null})=0$, the detour to $s_{null}$ does not lead to any discounting nor any addition reward. The only effect is the prolonged episode lengths.
	
	On the other hand, we use a coupling argument to prove $\rho^+_\pi \propto \rho_\pi$ over $\SS$. Consider the  \emph{coupled sampling} procedure shown in Algorithm \ref{algo_sampling}. For ease of notation we use ``null'' to denote the auxiliary state $s_{null}$ in the rest of the proof. Consider the status of the variables in the procedure at an arbitrary time $t>0$. $s_t$ is simply a regular sample of the original model $M_\pi$, so $s_t \sim \rho_\pi^{(t)}$. 
	
	$\zeta^+$ is obtained by inserting with probability $\epsilon$ a null state after each terminal state in the original rollout trajectory $\zeta=\{s_t\}$. Comparing with Definition \ref{def_perturbed_model}, we see that $\zeta^+$ follows the perturbed model $(M^+_\pi,\rho_0)$ under $\pi$. More accurately, let $s^+_t$ denote the state in $\zeta^+$ at time $t$, we have $s^+_t \sim \rho^{+(t)}_\pi$. 
	
	$z_t$ always equals an old state that $\zeta$ has encountered at an earlier time, with $\Delta_t$ being the time difference, so $z_t$ (as a random variable well defined by Algorithm \ref{algo_sampling}) must follow the same marginal distribution with $s_{t-\Delta_t}$, thus we have $z_t \sim \rho_\pi^{(t-\Delta_t)}$.
	
	\begin{minipage}[t]{.55\linewidth}
		\vspace{0pt}
		\centering
		\begin{algorithm}[H] \small
			\caption{a coupled sampling procedure}
			\label{algo_sampling}
			\KwIn{$M_\pi$,~  $\rho_0$,~  an i.i.d. sampler $random() \sim \U[0,1]$}
			\KwOut{an infinite tajectory $(z_0,z_1 \dots)$}
			
			sample $s_0 \sim \rho_0$ \\
			initialize trajectory $\zeta^+ \gets (s_0)$ \\
			set $\Delta_0 \gets 0$ \\
			set $z_0 \gets s_0$ \\
			\For{$t = 1,2,\dots$} 
			{
				sample $s_t \sim M_\pi(s_{t-1})$ \\
				\If{$s^+_{t-1} \in \SS_\perp$ {\bf and} $random()<\epsilon$}{
					append $(s_{null}, s_t)$ to $\zeta^+$
				} \Else{
					append $s_t$ to $\zeta^+$
				}
				set $\Delta_t \gets \# s_{null}$ in subsequence $\zeta^+_{0:t-1}$ \\
				set $z_t \gets s_{t-\Delta_t}$
			}
		\end{algorithm}
	\end{minipage}
	\hfill
	\begin{minipage}[t]{0.4\linewidth}
		\vspace{3pt}
		\centering
		\begin{tabular}{cccc}
			\hline
			$\zeta$		& $\zeta^+$			& $\Delta_t$	& $\{z_t\}$ \\	
			\hline
			$s_0$		& $\highlight{s_0}$	& $0$			& $s_0$ 	\\
			$s_1$		& null 				& $0$ 			& $s_1$		\\
			$s_2$		& $\highlight{s_1}$	& $1$			& $s_1$		\\
			$s_3$		& $\highlight{s_2}$	& $1$			& $s_2$		\\
			$s_4$		& null				& $1$			& $s_3$		\\
			$s_5$		& $s_3$				& $2$			& $s_3$		\\
			& $\highlight{s_4}$	& 				&			\\
			& $s_5$				&				&			\\
			&					&				&			\\
			\hline
		\end{tabular} 
		\flushleft
		\vspace{3pt} 
		\small
		An example running of Algorithm \ref{algo_sampling}. The table shows a snapshot when $t=5$. In the $\zeta^+$ column, ``null'' denotes $s_{null}$, and terminal states ($s_0, s_1, s_2, s_4$) are highlighted with gray background.
	\end{minipage}
	%

	In above we have associated $\rho\step{t}_\pi$ and $\rho^{+{\scriptscriptstyle (t)}}_\pi$ to the state distributions of the three random sequences $\zeta$, $\zeta^+$, and $\{z_t\}$, now we consider the coupling effects between these random sequences. Observe that when $s^+_t$ is not null, it must be an old state in $\zeta$ with time index shifted by the number of null states inserted before $t$ (i.e. by $t-1$) in previous samplings, which is \emph{the} state the procedure uses to assign value for $z_t$. In other words, $z_t=s^+_t$ conditioned on $s^+_t\neq$ null. Therefore, for any state $s\in \SS$,
	\eq{z_s_plus}{
		\Pr{z_t = s} = \Pr{s^+_t=s | s^+_t\neq \text{null}} 
		= \frac{\Pr{s^+_t=s} \* \Pr{s^+_t=\text{null}|s^+_t=s}}{\Pr{s^+_t\neq \text{null}}}
		= \frac{\Pr{s^+_t=s} \* 1}{\Pr{s^+_t\neq \text{null}}}
		.}
	
	\eqref{z_s_plus} holds for any $t>0$, thus must also hold at limit when $t\ra \infty$, as long as such limits exist. As $s^+_t$ follows the perturbed Markov chain $M^+_\pi$, it is known to have limiting distribution as proved in Theorem \ref{thm_limiting_dist}, thus the limits of both $\Pr{s^+=s}$ and $\Pr{s^+_t\neq \text{null}}$ at the rhs of \eqref{z_s_plus} exist, which means the limit of the lhs $\limit{t \ra \infty} \Pr{z_t=s}$ must also exist. Let $c_\pi \define \limit{t\ra\infty}\Pr{s^+_t \neq \text{null}}$, we have
	\eq{z_limit}{
		\limit{t\ra\infty} \Pr{z_t=s} = \limit{t\ra\infty} \Pr{s^+_t=s} / c_\pi = \rho^+_\pi(s) / c_\pi
		.}
	Note that $c_\pi>0$ because $\epsilon<1$ by definition.
	
	Now we only need to prove $\limit{t\ra\infty} \Pr{z_t=s} = \rho_\pi$. For that purpose, first observe that $\Delta$ gets increased in Algorithm \ref{algo_sampling} only if $\zeta^+$ entered $s_{null}$ in the last step -- only in that case $\Delta$ has a new ``null'' counted in. In other words, $\Delta_{t+1}=\Delta_t+1$ if $s^+_t=\text{null}$, otherwise $\Delta_{t+1}=\Delta_t$.
	
	Now consider the value of $z_{t+1}$. When $s^+_t=\text{null}$, we have $z_{t+1} = s_{t+1-\Delta_{t+1}} = s_{t+1-\Delta_t-1} = s_{t-\Delta_t}$. When $s^+_t\neq \text{null}$, we have $z_{t+1} = s_{t+1-\Delta_{t+1}} = s_{t+1-\Delta_t}$. Let $c_\pi^{(t)} = \Pr{s^+_t\neq \text{null}}$ for any $t$, then for any $s\in \SS$ we have
	\eq{z_t_plus_1}{
		\Pr{z_{t+1}=s} = (1-c_\pi^{(t)}) \* \Pr{s_{t-\Delta_t}=s} + c_\pi^{(t)} \* \Pr{s_{t+1-\Delta_t}=s}
		.}
	Due to \eqref{z_limit}, the two terms $\Pr{z_{t+1}=s}$ and $\Pr{s_{t-\Delta_t}=s}$ (which equals $\Pr{z_t=s}$ by the definition of $z_t$) in \eqref{z_t_plus_1} have the same limit. Taking both sides of \eqref{z_t_plus_1} to limit and re-arranging, yields
	\eq{consecutive_step}{
		c_\pi \* \limit{t\ra\infty} \Pr{s_{t-\Delta_t}=s} 
		= \limit{t\ra\infty} \Pr{z_{t+1}=s} - (1-c_\pi) \limit{t\ra\infty} \Pr{z_t=s} 
		= c_\pi \* \limit{t\ra\infty} \Pr{s_{t+1-\Delta_t}=s}
		.}
	The two ends of \eqref{consecutive_step} gives $\limit{t\ra\infty} \Pr{s_{t-\Delta_t}=s} = \limit{t\ra\infty} \Pr{s_{t+1-\Delta_t}=s}$, which holds for all states $s\in \SS$, thus by the definition of $\rho\step{t}_\pi$ we have
	\eq{rho_recursion}{
		\limit{t\ra\infty} \rho_\pi^{({t-\Delta_t})} = 
		\limit{t\ra\infty} \rho_\pi^{({t+1-\Delta_t})} = 
		\limit{t\ra\infty} \rho_\pi^{({t-\Delta_t})} \* M_\pi
		.}
	The two ends of \eqref{rho_recursion} gives a fixed point of the operator $M_\pi$, for which $\rho_\pi$ is known to be the only solution, so $\limit{t\ra\infty} \rho_\pi^{({t-\Delta_t})} = \rho_\pi$. Further combining with \eqref{z_limit}, we finally obtain
	\eq{}{
		\rho_\pi(s) = \limit{t\ra\infty} \Pr{s_{t-\Delta_t}=s} = \limit{t\ra\infty} \Pr{z_t=s} = \rho^+_\pi(s) / c_\pi
		,}
	for every $s\in \SS$.
\end{proof}

\subsection{Theorem \ref{thm_avg_epi}}
\label{proofsec_thm_avg_epi}
\begin{proof}
	Averaging the both sides of the Bellman equation \eqref{bellman} over the stationary distribution $\rho_pi$ and re-arranging, yields
	\eqmx{
		0 
		&= \Exp{s\sim \rho_\pi \\ a\sim \pi(s)}{ 
			\E{s'\sim P(s,a) \\ a'\sim\pi(s')} [R(s')+\gamma(s')Q_\pi(s',a')] -Q_\pi(s,a) 
		} \\
		&= \E{s\sim \rho_\pi \\ a\sim \pi(s)} \Exp{s'\sim P(s,a) \\ a'\sim\pi(s')}{R(s')} +
		\E{s\sim \rho_\pi \\ a\sim \pi(s)} \Exp{s'\sim P(s,a) \\ a'\sim\pi(s')}{\gamma(s')Q_\pi(s',a')} -
		\Exp{s\sim \rho_\pi \\ a\sim \pi(s)}{Q_\pi(s,a)} \\
		&= \E{s\sim \rho_\pi} [R(s)] + \E{s\sim \rho_\pi \\ a\sim \pi(s)} [\gamma(s)Q_\pi(s,a)] - \E{s\sim \rho_\pi \\ a\sim \pi(s)} [Q_\pi(s,a)] \\
		&= \E{s\sim \rho_\pi} [R(s)] - \E{s\sim \rho_\pi \\ a\sim \pi(s)} [\big( 1-\gamma(s) \big) Q_\pi(s,a)]
	}
	or equivalently
	\eq{R_Q}{
		\Exp{s\sim\rho_\pi}{R(s)} = \E{s\sim\rho_\pi}~ \Exp{a\sim\pi(s)}{\Big( 1-\gamma(s) \Big) \* Q_\pi(s,a)}
	.}

	Substituting $\ga(s)=\indicator{s\not\in \terminals}$ into \eqref{R_Q} will remove all terms corresponding to non-terminal states, giving	
	\eqm{R_Q0}{
		\E{s\sim \rho_\pi} [R(s)] &= \sum_{s\in \SS_\perp} \rho_\pi(s) \* \E{a\sim\pi(s)}[Q_\pi(s,a)] \\
		&= \Big( \sum_{s\in \SS_\perp} \rho_\pi(s) \Big) \* V_\pi(s_\perp) 
		= \Big( \sum_{s\in \SS_\perp} \rho_\pi(s) \Big) \* J_{epi} 
	} 	
	Then the following proposition transforms $\sum_{s\in \SS_\perp} \rho_\pi(s)$ into $1/\E{\zeta\sim M_\pi}[T(\zeta)]$, as desired.
	
	\begin{proposition}
	\label{prop_ET}
		$\E{\zeta\sim M_\pi}[T(\zeta)] = 1 / \sum_{s\in \SS_\perp} \rho_\pi(s)$
	\end{proposition}
	\begin{proof} 
	Consider the MDP $\M'$ obtained by grouping all terminal states in $\M$ into a single ``macro-state'' $s_\perp$. Note that in general, grouping terminal states into one will change how the model delivers the rewarding feedback to the agent, as $R(s)$ may vary between terminal states. However, for the purpose of this proof, the transition dynamics of $\M$ and $\M'$ under \emph{given} policy $\pi$ (that is, the markov chains $M_\pi$ and $M'_\pi$) will remain the same as all terminal states are homogeneous in transition probabilities. In particular, the expected episode length of $M_\pi$ will be the same with that of $M'_\pi$, from which we obtain 
	\eqm{}{
		\E{\zeta\sim M_\pi}[T(\zeta)] &= \E{\zeta\sim M'_\pi}[T(\zeta)] = \E{\zeta\sim M'_\pi}[T_{s_\perp}(\zeta)] \\ 
		&= 1/\rho'(s_\perp) = 1/(1-\sum_{s\in \SS \setminus \terminals} \rho'(s)) \\
		&= 1/(1-\sum_{s\in \SS \setminus \terminals} \rho(s)) = 1/\sum_{s\in \SS_\perp} \rho_\pi(s)
	.}
	In above $\E{}[T]=\E{}[T_{s_\perp}]$ because $s_\perp$ is the only terminal states in $M'_\pi$, and $\rho'(s)$ exists because $\M'$ is episodic.  
	\end{proof}
	Substituting Proposition \ref{prop_ET} back to \eqref{R_Q0} completes the proof.
\end{proof}

\subsection{Theorem \ref{thm_pg}}
\label{proofsec_thm_pg}
\begin{proof}
	First consider the quantity $\Exp{s,a\sim \rho_\theta}{\nabla_\theta Q_\theta (s,a) }$ in its general form (not necessarily with the specific episodic $\gamma$ function as assumed). By the Bellman equation \eqref{bellman} we have
	\eqm{avg_grad_q}{
		&\Exp{s,a\sim \rho_\theta}{\nabla_\theta Q_\theta (s,a) }
		= \Exp{s,a\sim \rho_\theta} {
			\nabla_\theta \Exp{s'\sim \T(s,a)}{\R(s') + \gamma(s') \sum_{a'} \pi(a'|s';\theta) \, Q_\theta(s',a')}
		} \\
		&= \E{s,a\sim \rho_\theta} \Exp{s'\sim \T(s,a)}{
			\gamma(s') \sum_{a'} \nabla_\theta \Big( \pi(a'|s';\theta) \, Q_\theta(s',a') \Big)
		} \\
		&= \Exp{s \sim \rho_\theta}{
			\gamma(s) \sum_{a} \nabla_\theta \Big( \pi(a|s;\theta) \, Q_\theta(s,a) \Big)
		} \\
		&= \Exp{s\sim \rho_\theta}{
			\gamma(s)\sum_a  \Big( Q_\theta(s,a) \nabla_\theta \pi(a|s;\theta) + \pi(a|s;\theta) \nabla_\theta Q_\theta(s,a) \Big)
		} \\
		&= 	\Exp{s,a\sim \rho_\theta}{\gamma(s) Q_\theta(s,a)\nabla_\theta \log \pi(a|s;\theta)} + 
		\Exp{s,a\sim \rho_\theta}{\gamma(s) \nabla_\theta Q_\theta(s,a)}
		.}
	Moving the second term in the right-hand side of \eqref{avg_grad_q} to the left, yields
	\eq{avg_grad_q_2}{
		\Exp{s,a\sim \rho_\theta}{\Big(1-\gamma(s)\Big) \nabla_\theta Q_\theta(s,a)} = 
		\Exp{s,a\sim \rho_\theta}{\gamma(s) \, Q_\theta(s,a)\nabla_\theta \log \pi(a|s;\theta)}
	} 
	Now consider the specific $\gamma$ function with $\gamma(s)=1$ for $s\not\in \SS_\perp$, and $\gamma(s)=0$ for $s\in \SS_\perp$. Substituting such episodic $\gamma$ function into \eqref{avg_grad_q_2}, yields
	\eq{avg_grad_q_3}{
		\sum_{s\in \SS_\perp} \rho_\theta(s) \Exp{a\sim\pi(s;\theta)}{\nabla_\theta Q_\theta(s,a)} 
		= \sum_{s\not\in \SS_\perp} \rho_\theta(s) \Exp{a\sim\pi(s;\theta)}{Q_\theta(s,a) \nabla_\theta \log\pi(s,a;\theta)}
		.}
	The left-hand side of \eqref{avg_grad_q_3} is an average of $\nabla_\theta Q_\theta(s,a)$ over terminal states $s\in \terminals$ and over the policy-induced actions $a\sim \pi(s;\theta)$. Note that $\nabla_\theta Q_\theta(s_\perp,a)=\nabla_\theta J_{epi}(\theta)$ for any $s_\perp \in \terminals$ and any $a\in \AS$, which means in the left-hand side we are averaging a constant term that can be moved out of the summations (of both $s$ and $a$), yielding 
	\eq{avg_grad_q_4}{
		\nabla_\theta J_{epi} \* \sum_{s\in \SS_\perp} \rho_\theta(s)
		= \sum_{s\not\in \SS_\perp} \rho_\theta(s) \Exp{a\sim\pi(s;\theta)}{Q_\theta(s,a) \nabla_\theta \log\pi(s,a;\theta)}
		.}
	
	Denoting $\rho_\theta(\SS_\perp) \define \sum_{s\in \SS_\perp} \rho_\theta(s)$ and $\rho_\theta(\SS \setminus \SS_\perp) \define \sum_{s\not\in \SS_\perp} \rho_\theta(s)$, and substituting back to \eqref{avg_grad_q_4}, yields
	\eqm{}{
		\nabla_\theta J_{epi}
		&= \frac{1}{\rho_\theta(\SS_\perp)} \* \sum_{s\not\in \SS_\perp} \rho_\theta(s) \, \Exp{a\sim\pi(s;\theta)}{Q_\theta(s,a) \nabla_\theta \log\pi(s,a;\theta)} \\
		&= \frac{\rho_\theta(\SS \setminus \SS_\perp)}{\rho_\theta(\SS_\perp)} \* \sum_{s\not\in \SS_\perp} \frac{\rho_\theta(s)}{\rho_\theta(\SS \setminus \SS_\perp)} \, \Exp{a\sim\pi(s;\theta)}{Q_\theta(s,a) \nabla_\theta \log\pi(s,a;\theta)} \\
		&= \frac{1-\rho_\theta(\SS_\perp)}{\rho_\theta(\SS_\perp)} \* \Exp{s,a \sim \rho_\theta}{Q_\theta(s,a) \, \nabla_\theta \log \pi(s,a;\theta) ~\Big|~ s \not\in \SS_\perp}
	}
	Finally, due to Proposition \ref{prop_ET}, $\frac{1-\rho_\theta(\SS_\perp)}{\rho_\theta(\SS_\perp)} = \E{\theta}[T]-1$. 
\end{proof}

\section{On the definition of on-policy distribution in episodic tasks}
\label{sec_on_policy}
In page 199 of \citet{2018:RL}, the on-policy distribution of episodic tasks has been defined as $\mu^1_\pi$, the normalized undiscounted expected visiting frequency of a state in an episode. We believe the stationary distribution $\rho_\pi$ identified in this paper helps with better shaping this conception in at least three aspects:

First, defining on-policy distribution directly as the stationary distribution $\rho_\pi$ helps unify the notion of on-policy distribution across continual and episodic tasks without changing its mathematical identity (as $\rho_\pi=\mu^1_\pi$).

Second, it seems that the conception of on-policy distribution is \emph{intended} to capture ``\emph{the number of time steps sent, on average, in state $s$ in a single episode}'' (\citet{2018:RL}, page 199). However, note that there are two possible formal semantics: $\E{\zeta} [n_s] / \E{\zeta} [T]$ and $\E{\zeta} [n_s/T]$ -- both capture the intuition of ``time spent on $s$ in an episode on average'', and it is not immediately clear why we should favor normalizing all episodes uniformly by the average episode length over normalizing in a per-episode manner. In fact, the ergodic theorem (i.e. Proposition \ref{thm_ergodic}) has ``favored'' the latter semantic, by connecting the per-episode-normalized reward $\E{}[\sum_{t=1}^T R(s_t)/T]$ to the steady-state reward $\E{s\sim \rho_\pi}[R(s)]$. In comparison, Theorem \ref{thm_avg_epi} ``favors'' the former semantic, establishing equality between the uniformly-normalized distribution $\mu^1_\pi$ and the steady-state distribution $\rho_\pi$, which justifies formal semantic of ambiguous intuition in a more principled way.

Lastly, in \citet{2018:RL} (page 199, Eq. 9.2), the ``existence'' of the on-policy distribution is defended by writing $\mu_\pi^1$ into the (normalized) solution of a non-homogeneous system of linear equations. But not every non-homogeneous linear system has a unique solution. Also, the linear-system argument becomes more subtle when generalizing from finite state spaces (which is assumed in \cite{2018:RL}) to infinite state spaces. Our treatment to the stationary distribution $\rho_\pi$ thus consolidates the concept of on-policy distribution by providing an alternative theoretical basis (based on the markov chain theory) for this term.

\section{On what traditional policy-gradient estimators actually compute}
\label{sec_pg_algo}

As common practice, traditional policy gradient algorithms use the discounted value function but compute the gradient based on undiscounted data distribution, which neither follows the classic (discounted) policy gradient theorem nor does it follow the undiscounted steady-state policy gradient theorem. This gap between theory and practice is well known in the community, and \cite{2014:thomas} has examined what such ``mixed'' policy gradient estimation would obtain in the continual settings. But similar analysis in the episodic case was not done exactly because of the lack of steady state for the latter. As \citet{2014:thomas} noted, ``\emph{[the steady-state objective] $\bar{J}$ is not interesting for episodic MDPs since, for all policies, [the steady-state distribution] $\bar{d}^\theta(s)$ is non-zero for only the post-terminal absorbing state. So, henceforth, our discussion is limited to the non-episodic setting}''.

However, the existence of unique stationary distribution now enables such analysis even for episodic tasks. Specifically, substituting $\ga(s)=\ga_c<1$ into \eqref{pg_proof1}, which is copied below for convenience
\begin{equation*}
\Exp{s,a\sim \rho_\theta}{\Big(1-\gamma(s)\Big) \nabla_\theta Q^{\ga_c}_\theta(s,a)} = 
\Exp{s,a\sim \rho_\theta}{\gamma(s) \, Q^{\ga_c}_\theta(s,a)\nabla_\theta \log \pi(a|s;\theta)}
,\end{equation*}
yields,
\eqm{}{
\Exp{s,a\sim \rho_\theta}{Q^{\ga_c}_\theta(s,a)\nabla_\theta \log \pi(a|s;\theta)} 
&= \frac{1-\ga_c}{\ga_c} \Exp{s,a\sim \rho_\theta}{\nabla_\theta Q^{\ga_c}_\theta(s,a)} 
}

The left-hand side is exactly what the classic policy gradient algorithms compute in practice, and the right-hand side is proportional to the gradient of the steady-state performance \emph{if} the policy change does not affect the stationary distribution. This observation is in analogy to what \citet{2014:thomas} concluded in the continual setting.


We note that, together with the proofs of Theorem \ref{thm_avg_epi} and \ref{thm_pg}, analysis in this section also exemplify a pattern of steady state analysis for episodic tasks: first examine properties of \emph{any} RL task under its steady state, then substitute in a specific $\gamma$ function to derive corollaries dedicated to episodic tasks as special cases. The same approach could potentially be used in more theoretical and algorithmic problems of episodic RL.

\section{Experiment details}
\label{sec_exp}

This section complements Section \ref{sec_algo} and reports details about the experimentation settings and results.
In all experiments, the behavior policy is a Gaussian distribution with diagonal covariance, whose mean is represented by a neural network with two fully-connected hidden layers of 64 units\cite{2017:ppo}. All model parameters are randomly initialized. The RoboSchoolBullet benchmark~\cite{pybullet} consists of challenging locomotion control tasks that have been used as test fields for state-of-the-art RL algorithms~\cite{2018:sac}.

\subsection{Details of the policy gradient experiments}
\label{sec_ael}

In the gradient checking experiment reported in Section \ref{sec_pg}, we run a standard stochastic gradient ascent procedure with the gradient estimated by $F_{SSPG}$. In each policy-update iteration, we run $K$ independent rollouts, each lasts for ten episodes. As our focus is to examine the quality of gradient estimation, we set the batch size $K$ to be one million, so as to have accurate estimate of the ground-truth gradient. We then compute the AEL term $\widehat{\E{\theta}[T]}-1$ by averaging over the $10\* K$ episodes, which is very close to the true AEL value (minus $1$) thanks to the very large sample size. The next step is to collect the $(s_t,a_t)$ pair at the single step of $t=3 \* AEL$ in each of the $K$ rollouts, creating an i.i.d. sample $\D$ of size $K$. The $\widehat{Q_\theta}(s,a)$ term for $(s,a)\in \D$ is estimated using the corresponding episode return (from $s$). 

Two gradient estimators are implemented: one follows exactly $F_{SSPG}$ with the AEL term, the other omits the AEL term from $F_{SSPG}$. We applied constant learning rate $\alpha=5\times 10^{-4}$ to the estimator \emph{omitting} the AEL term, which simulated what traditional policy gradient estimators have been doing. The resulted policy update from this baseline method is thus $\Delta \theta = \alpha \* \bar{F}_{SSPG} / (\E{}[T]-1)$, where $\bar{F}_{SSPG}$ denotes the mean value of $F_{SSPG}$ over the sample $\D$. We call such a ``standard practice'' of policy-gradient update, the \emph{constant learning rate} method. It is clear that the ``constant learning rate'' method is equivalent to applying a drifting learning rate $\alpha/(\E{}[T]-1)$ to the truly unbiased gradient estimator $\bar{F}_{SSPG}$.

On the other hand, we applied the same constant learning rate $\alpha$ also to the estimator \emph{with} the AEL term, computing policy update as $\Delta \theta = \alpha \* \bar{F}_{SSPG}$. This is equivalent to applying an AEL-adaptive learning rate $\alpha \* (\E{}[T]-1)$ to the traditional policy gradient estimator (that omits the AEL term), and is thus called the \emph{adaptive learning rate} method here. Note that the names of ``constant learning rate'' and ``adaptive learning rate'' are from the perspective of traditional policy gradient framework which omits the AEL term.

We measure the quality of an estimated policy gradient $\widehat{\nabla_\theta J}$ by examining the quality of the projected performance change $\widehat{\Delta J}$ as entailed by the estimated gradient. Specifically, let $\Delta\theta = \widehat{\nabla_\theta J} \* \alpha$ be the corresponding policy update of the estimated gradient, the projected performance change from such a policy update is calculated as $\widehat{\Delta J} \define ||\widehat{\nabla_\theta J}||_2 \* ||\Delta \theta||_2 = ||\widehat{\nabla_\theta J}||_2^2 \* \alpha$. It is known that for the $\widehat{\Delta J}$ thus computed, we have $\widehat{\Delta J}=\Delta J$ if $\widehat{\nabla_\theta J} = \nabla_\theta J$, and that the more biased the estimated policy gradient $\widehat{\nabla_\theta J}$ is (to the true gradient $\nabla_\theta J$), the more biased the projected performance change $\widehat{\Delta J}$ will be (to the true performance change $\Delta J$).
Based on this principle, we computed, for each policy-update iteration, the projected performance changes from both the ``constant learning rate'' method and the ``adaptive learning rate'' method, and compared them with the true performance change $\Delta J =J_{new}-J_{old}$ whose $\%90$-confidence interval is calculated from the statistics of episode-wise returns in the $10 \* K$ independent rollouts, for both the old and new policies. 

\begin{figure}[b]
	\centering
	
	\subfloat[Policy iteration $0-30$\label{fig:pg_0}]{
		\includegraphics[width=0.5\linewidth]{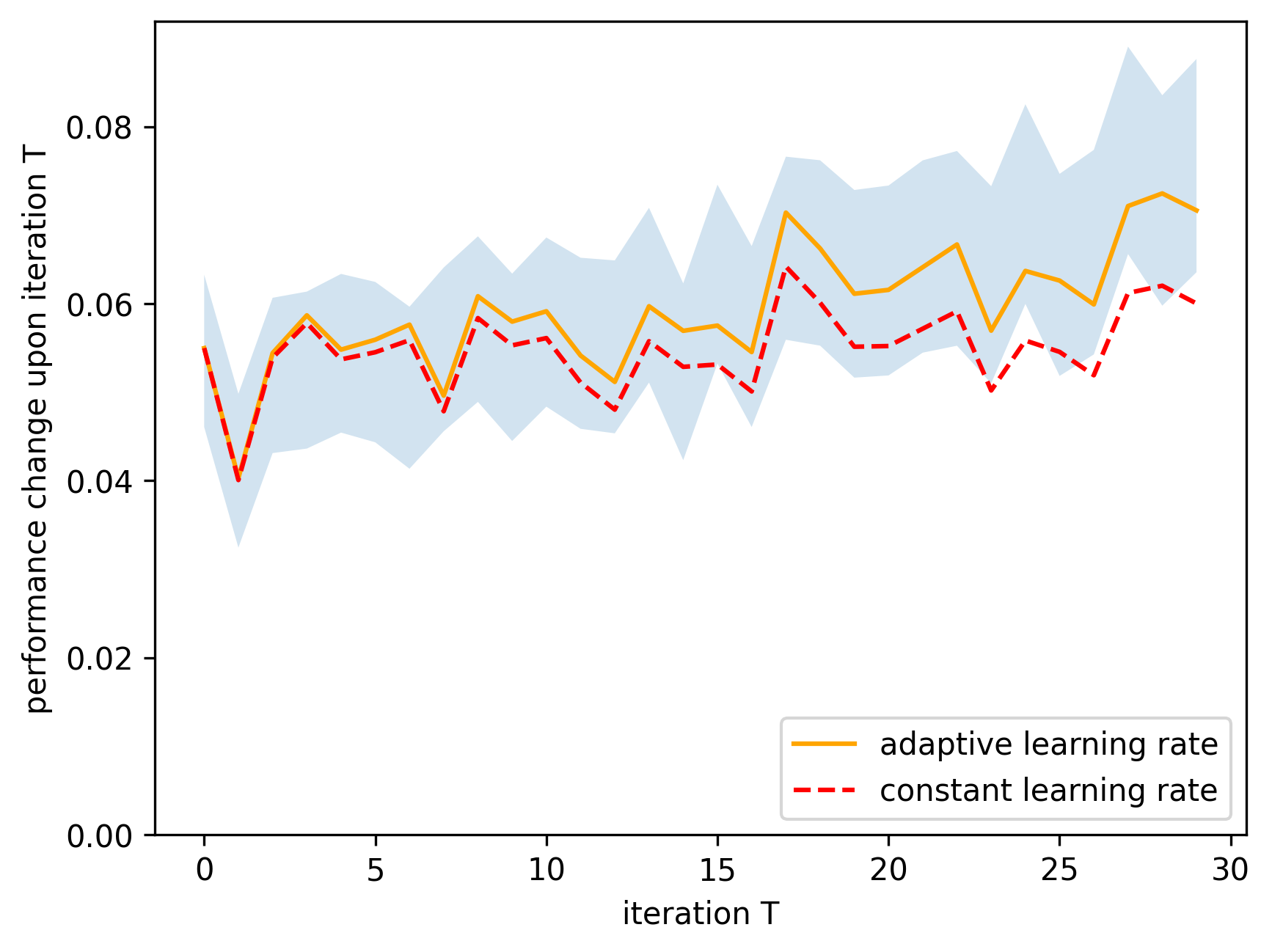}
	}
	\subfloat[Policy iteration $100-130$\label{fig:pg_100}]{
		\includegraphics[width=0.5\linewidth]{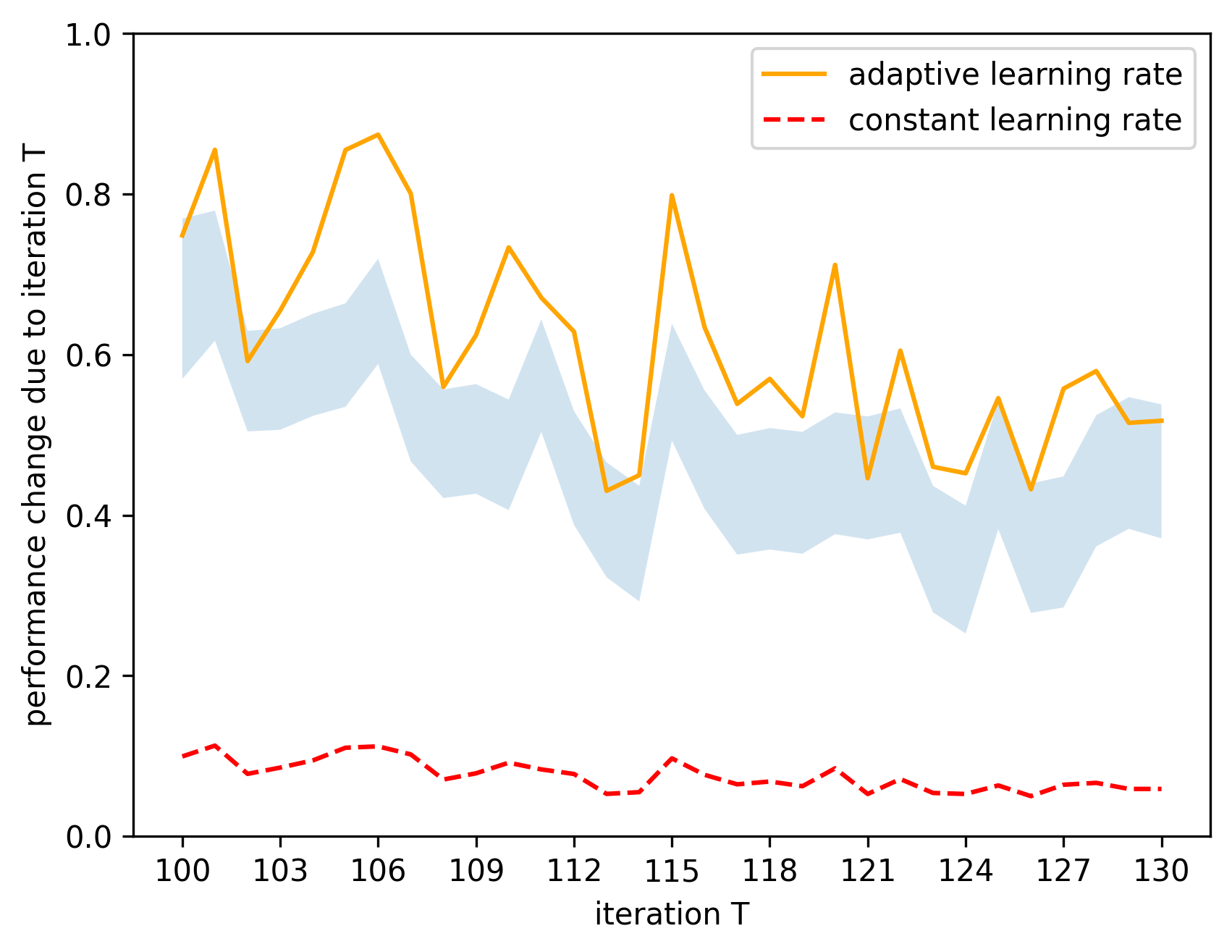}
	}
	\caption{Projected vs actual performance changes in HopperBulletEnv-v0 as quality checking for policy gradient estimators with the AEL term (corresponding to ``adaptive learning rate'') and without the AEL term (corresponding to ``constant learning rate''). 
	}
	\label{fig:hopper_reinforce}	
\end{figure}

We conducted the above experimentation procedure to the HopperBulletEnv-v0 environment in RoboSchoolBullet, and Figure \ref{fig:hopper_reinforce} revealed how quickly the drifting AEL term can hurt the quality of gradient estimation in this environment. The shaded area illustrates the $90\%$ confidence intervals of the true performance changes after each policy update. The red dotted curve is the projected performance change from the ``constant learning rate'' method which treats the AEL term as ``a proportionality constant that can be absorbed in the learning rate''. We see that this traditional policy gradient method quickly leads to bias after only tens of iterations, as Figure \ref{fig:hopper_reinforce}(a) shows. The bias becomes quite significant after $100$ updates, as Figure \ref{fig:hopper_reinforce}(b) shows. On the other hand, the orange curves are the projected performance change by the ``adaptive learning rate'' method, which follows the unbiased estimator as given by Theorem 5, and leads to much less bias as Figure \ref{fig:hopper_reinforce}(a) and Figure \ref{fig:hopper_reinforce}(b) show.

\subsection{Details of the perturbation experiments}
\label{sec_perturb_exp}

Figure \ref{fig:mixing_hopper} shows how the mean value of the marginal distribution for each state dimension evolves over time in the (multi-episode) learning process of the \emph{raw} Hopper environment under random policy, without perturbation. Each mean-value point in the figure is calculated by averaging over 100,000 rollouts, which serves as an index, or an indicator to the marginal distribution for the corresponding state dimension. Time is normalized to multiples of Average Episode Length (or AEL), and we see that the marginal distributions of all state dimensions have converged at $t= 2 \* AEL$.
We found that similar convergence rates apply to policies at different stages of the RL training in Hopper. The result suggests that perturbation may not be needed for the Hopper environment. 

On the other hand, Figure \ref{fig:mixing_humanoid_raw} and \ref{fig:mixing_halfcheetah_raw} show that in some other environments, the marginal distributions may converge much slower: In the Humanoid environment, the marginal distribution of the state dimensions takes more than $20\* AEL$ steps to converge, while in the HalfCheetah environment, the marginal distribution appears to not converge at all. Both environments indeed have strong periodicity in episode length. Especially, the episode length of Halfcheetah is actually fixed at exactly $1000$ under the random policy (and under any other policy as well). These observations indicate that perturbation is indeed needed in general, if we want to approximate the stationary distributions with marginal distributions of a single (or a few) step.

Moreover, Figure \ref{fig:mixing_humanoid_perturbed} and \ref{fig:mixing_halfcheetah_perturbed} illustrate that the rollouts in Humanoid and HalfCheetah quickly converge to their respective steady states after applying the recursive perturbation trick with $\epsilon=1-1/\E{}[T]$. Comparing with Figure \ref{fig:mixing_humanoid_raw} and \ref{fig:mixing_halfcheetah_raw}, respectively, we clearly see the effectiveness of the recursive perturbation on these two popular RL environments. In both cases the convergence occurs before $t^* = 3 \E{}[T]$.

Figure \ref{fig:mixing_synthetic} shows that the ``$3$-AEL convergence'' observation generalizes to even adversarially synthesized environments. These environments have deterministic episode lengths $n$ and the state $s_t$ regularly goes over from $0$ to $n-1$ in an episode. Without perturbation, the marginal distribution $\rho\step{t}$ of such environment would concentrate entirely on the single state $s_t = (t\mod n)$. We then applied the recursive perturbation to such ``state-sweeping'' environment with $n=20, 100, 500, 2000$. In all cases $\epsilon = 1-1/n$, and we run a large number of independent (and perturbed) rollouts for marginal distribution estimation (we have to run 30 million rollouts when $N=2000$ so as to observe the ground-truth distribution). From each of the rollouts we collected the two states at $t=3n$ and $t=3n-1$ as two sample points, and we observed the empirical state distribution over all the sample points thus obtained, for each environment (i.e. for each $n$). As we can see from Figure \ref{fig:mixing_synthetic}, even for the completely periodic environment with fixed episode length $n=2000$, the marginal distribution still converges well to the stationary distribution (which is the uniform distribution over $\{0\dots n-1\}$) in only $3n$ steps, after applying the recursive perturbation.

Note that the perturbation with self-loop probability $\epsilon=1-1/n$ causes the rollout to stay in the null state for $n-1$ steps per episode on average, which in turn causes half of the samples obtained at a fixed rollout time to be the null state (if the marginal distribution at that time has already converged). We sampled two consecutive time steps around $3n$ to compensate this loss of data, so that the two-step sampling provides roughly the same amount of ``useful samples'' as the original batch size (in one-step sampling without perturbation). Again, as Figure \ref{fig:mixing_synthetic} illustrates, the empirical distribution from such two-step sampling approximates the desired uniform distribution pretty well around $t^*=3n$. In other words, the practical cost for the half amount of samples wasted in the null state is to just sample one more step in the rollout.  

\begin{figure}[t]
	\centering
	\subfloat[HopperBulletEnv-v0 (raw)\label{fig:mixing_hopper}]{
		\includegraphics[width=0.5\linewidth]{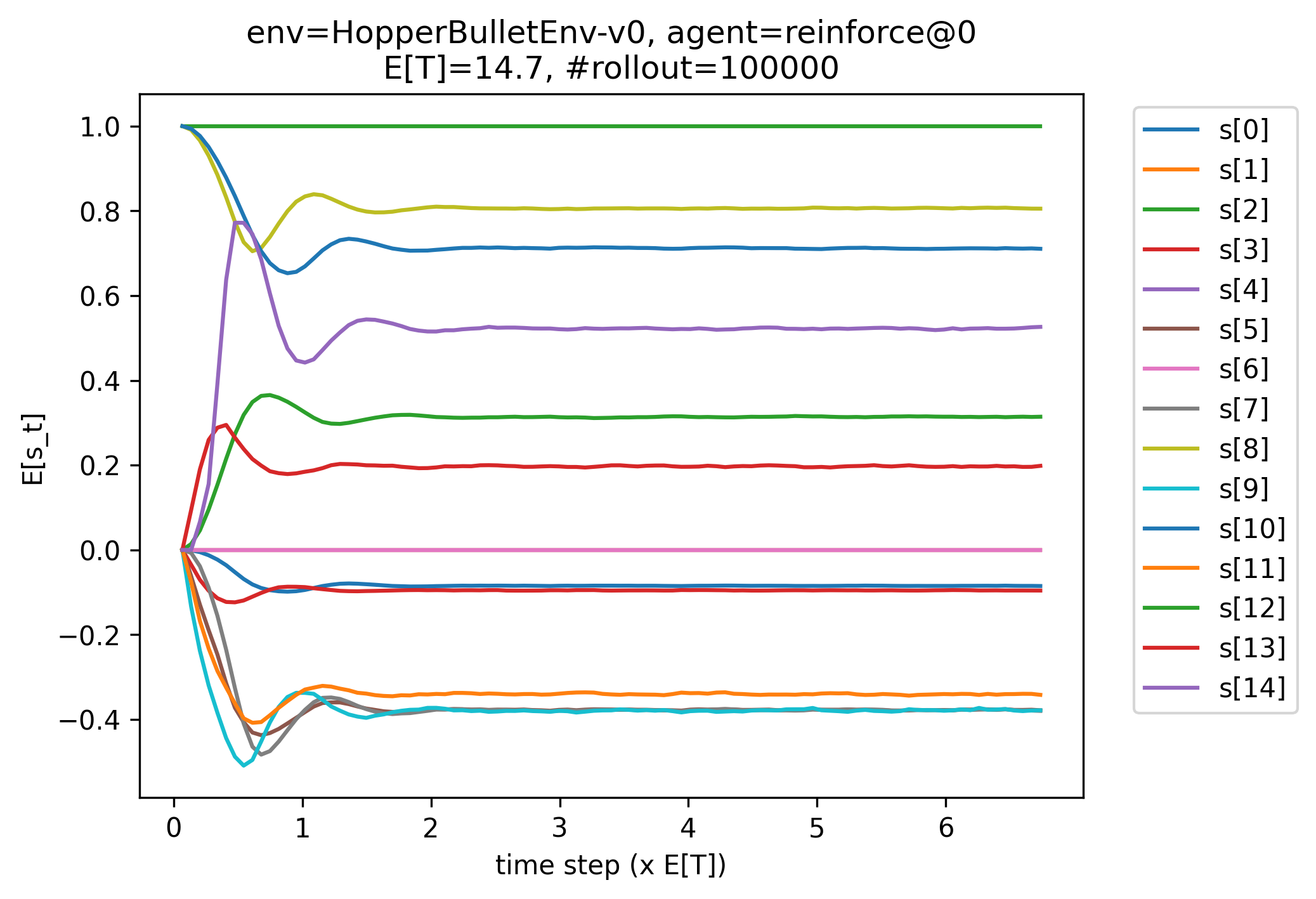}
	}
	\subfloat[State-Sweeping (perturbed)\label{fig:mixing_synthetic}]{
		\includegraphics[width=0.42\linewidth]{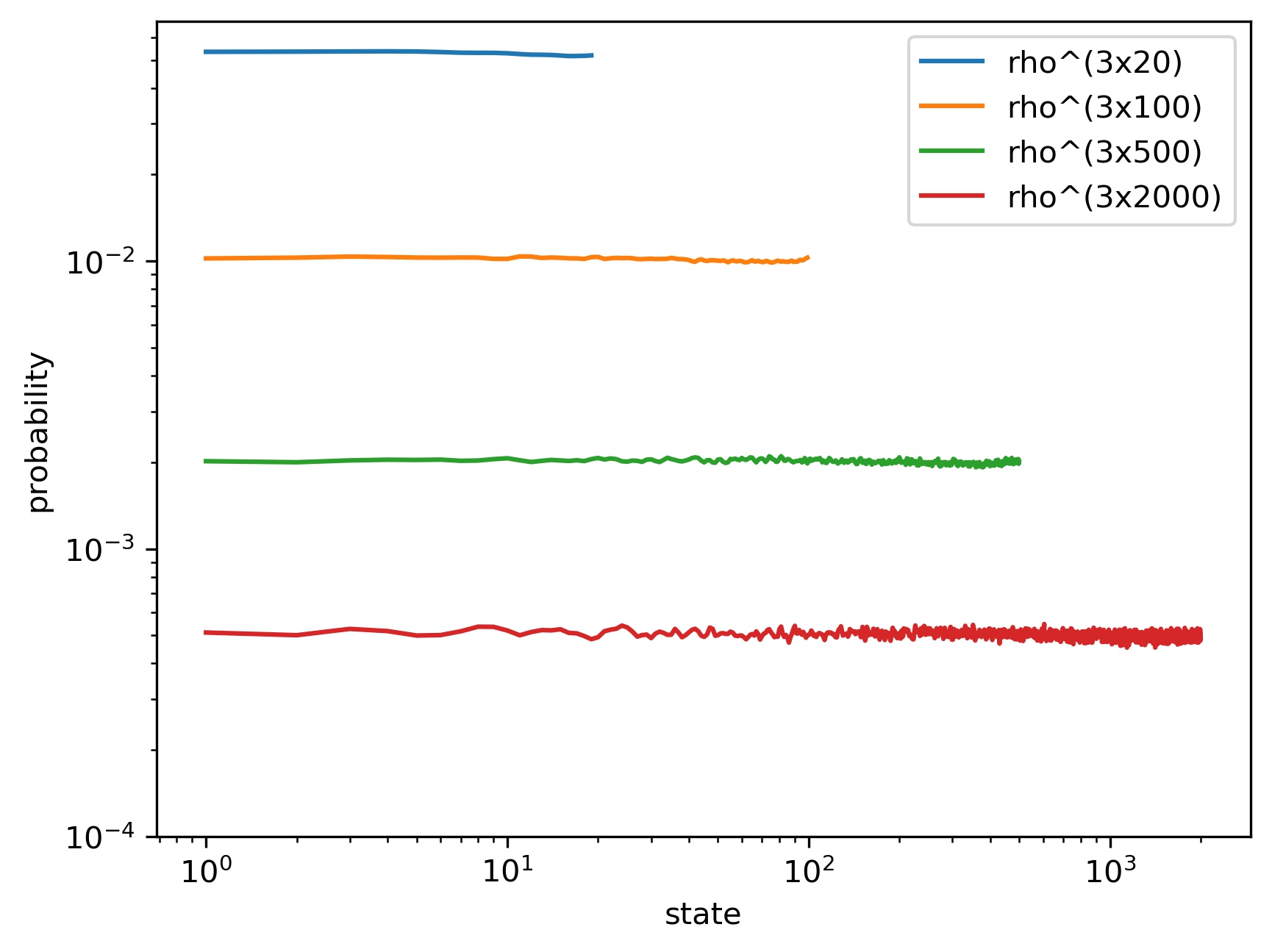}
	}
	\hfill
	\subfloat[HumanoidBulletEnv-v0 (raw)\label{fig:mixing_humanoid_raw}]{
		\includegraphics[width=0.5\linewidth]{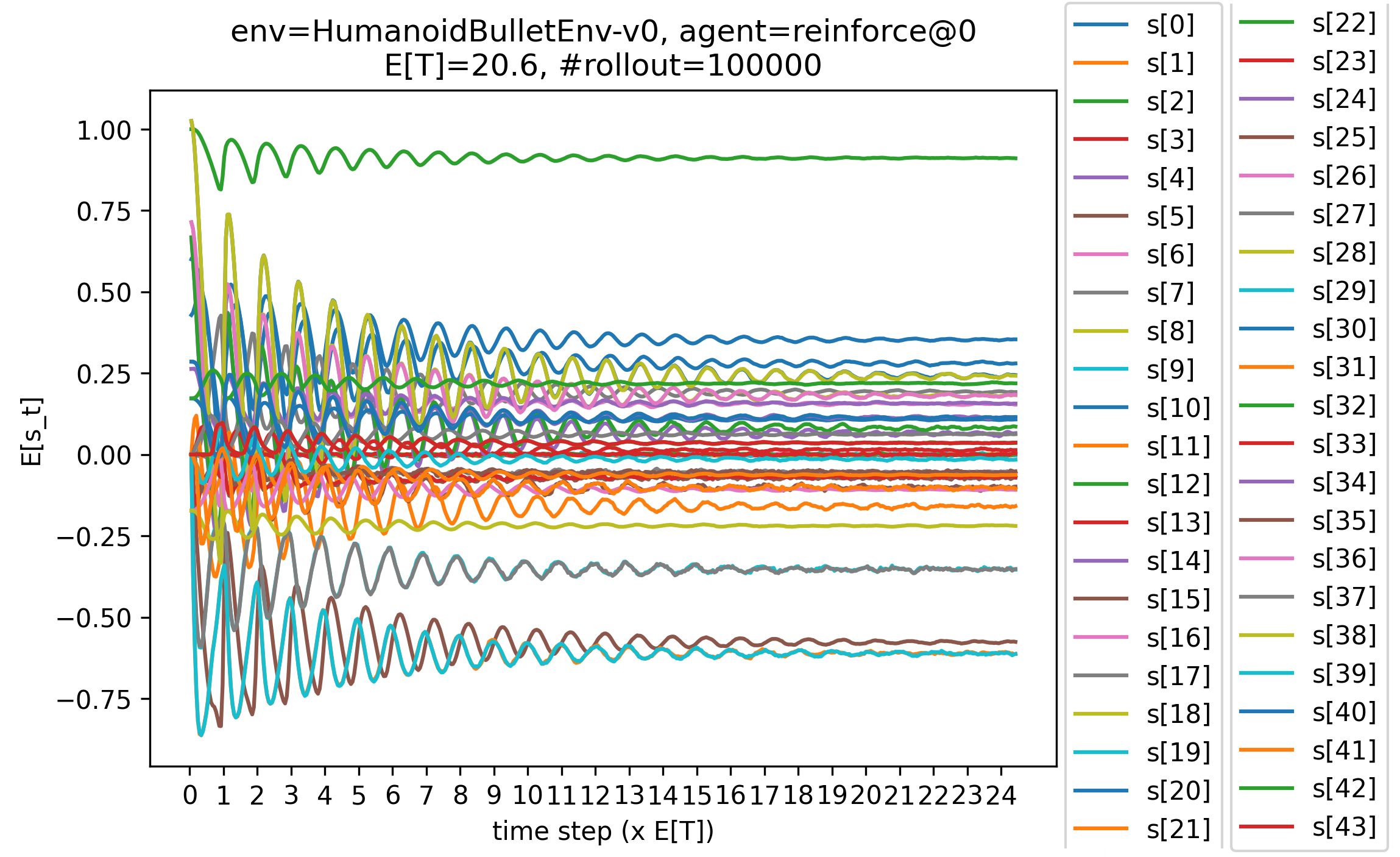}
	}
	\subfloat[HalfCheetahBulletEnv-v0 (raw)\label{fig:mixing_halfcheetah_raw}]{
		\includegraphics[width=0.5\linewidth]{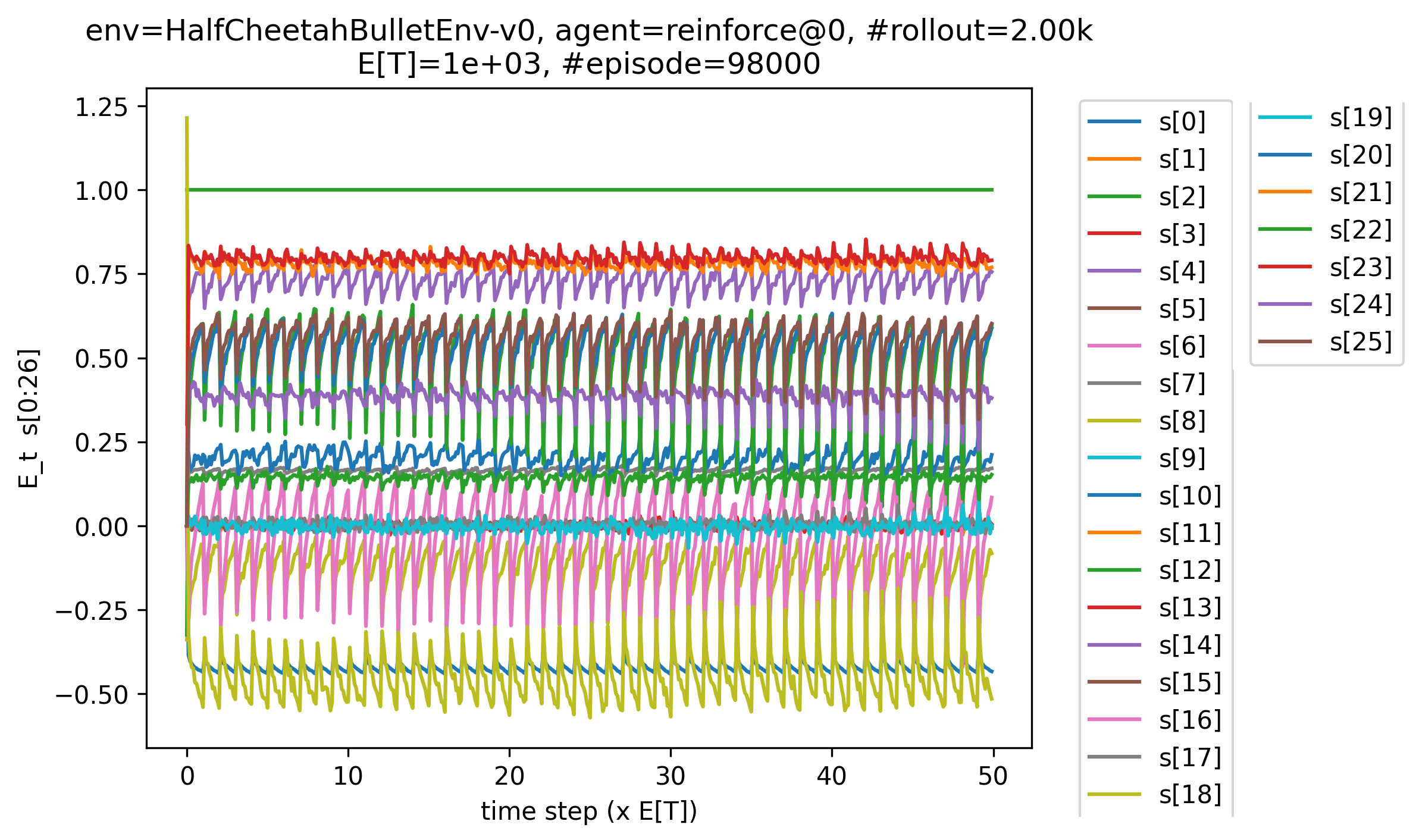}
	}
	\hfill
	\subfloat[HumanoidBulletEnv-v0 (perturbed)\label{fig:mixing_humanoid_perturbed}]{
		\includegraphics[width=0.5\linewidth]{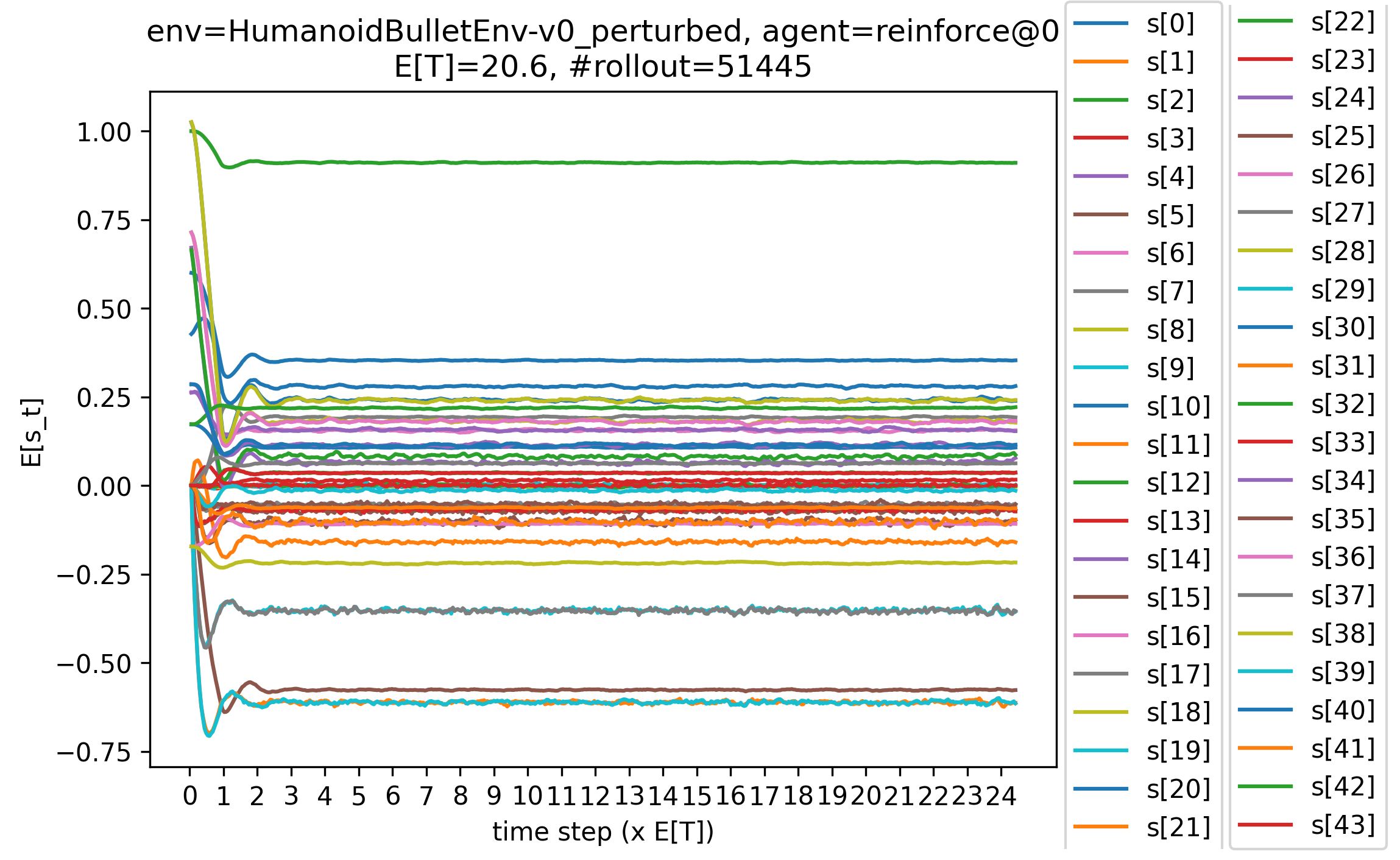}
	}
	\subfloat[HalfCheetahBulletEnv-v0 (perturbed)\label{fig:mixing_halfcheetah_perturbed}]{
		\includegraphics[width=0.5\linewidth]{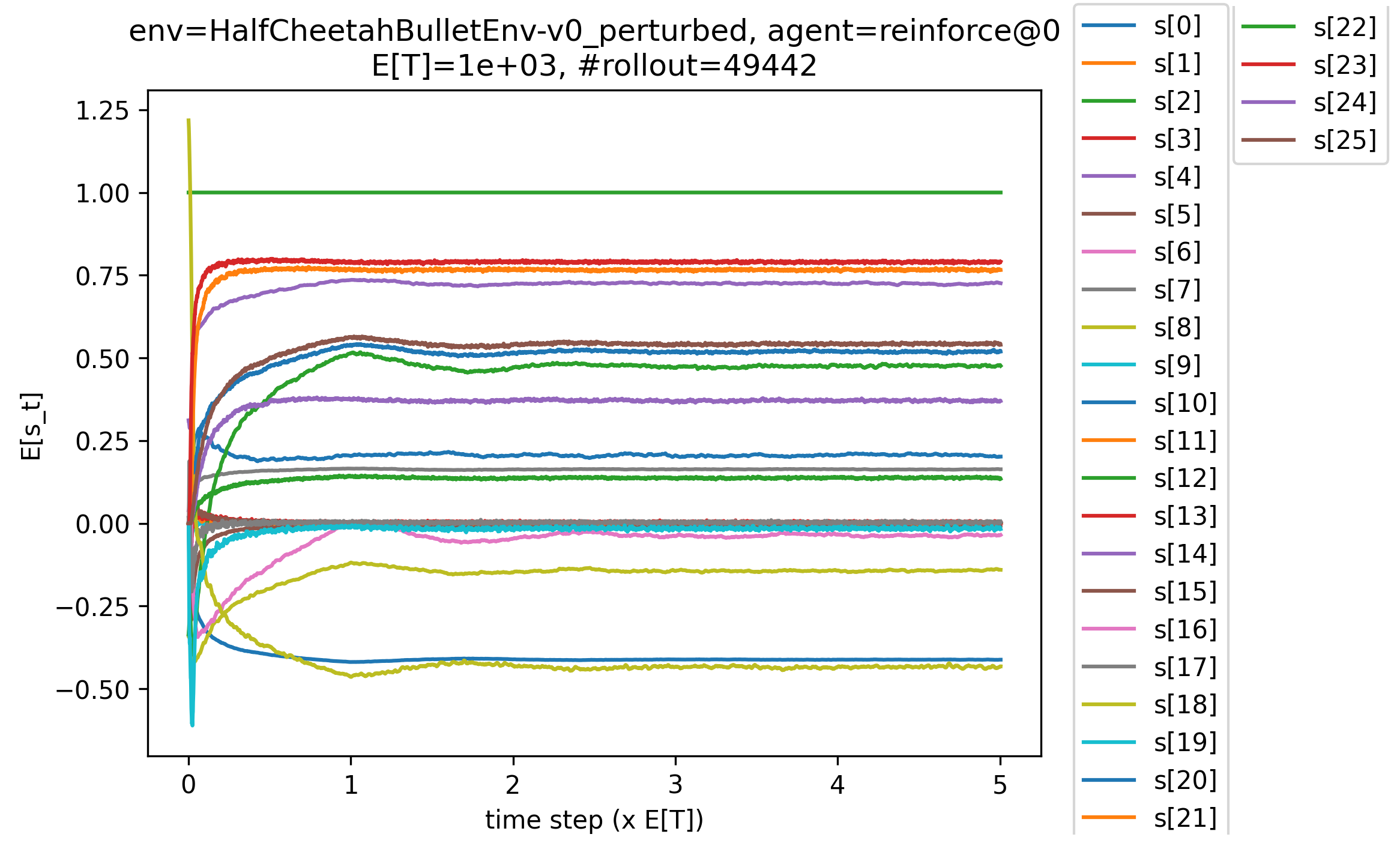}
	}
	\flushleft
	\caption{
		The ``$3$-AEL convergence'' phenomenon
	}
	\label{fig:mixing}
\end{figure}

\end{document}